\documentclass{article}

\usepackage{natbib}
\setcitestyle{numbers,square}
\usepackage[final]{neurips_2024}
\usepackage[utf8]{inputenc} % allow utf-8 input
\usepackage[T1]{fontenc}    % use 8-bit T1 fonts
\usepackage{hyperref}       % hyperlinks
\usepackage{url}            % simple URL typesetting
\usepackage{booktabs}       % professional-quality tables
\usepackage{amsfonts}       % blackboard math symbols
\usepackage{nicefrac}       % compact symbols for 1/2, etc.
\usepackage{microtype}      % microtypography
\usepackage{xcolor}         % colors
\usepackage{amssymb}
\usepackage{mathtools}
\usepackage{amsthm}
\usepackage{amsmath}
\usepackage{multirow}
\usepackage{adjustbox}
\usepackage{bm}
\usepackage{thmtools}
\usepackage{thm-restate}
\usepackage{subcaption} 
\usepackage{hyperref} % For \texorpdfstring
\DeclareMathOperator*{\argmax}{arg\,max}
\DeclareMathOperator*{\argmin}{arg\,min}

\usepackage[capitalize,noabbrev]{cleveref}
\crefname{assumption}{Assumption}{Assumptions}

\theoremstyle{plain}
\newtheorem{theorem}{Theorem}[section]
\newtheorem{proposition}[theorem]{Proposition}
\newtheorem{lemma}[theorem]{Lemma}

\theoremstyle{definition}

\newtheorem{assumption}[theorem]{Assumption}
\theoremstyle{remark}

\title{Is Score Matching Suitable for Estimating Point Processes?}

\author{Haoqun Cao$^1$, Zizhuo Meng$^2$, Tianjun Ke$^1$, Feng Zhou$^{1,3}$\thanks{Corresponding author.}\\
$^1$Center for Applied Statistics and School of Statistics, Renmin University of China\\
$^2$Data Science Institute, University of Technology Sydney\\
$^3$Beijing Advanced Innovation Center for Future Blockchain and Privacy Computing\\
\texttt{hcao65@wisc.edu, feng.zhou@ruc.edu.cn}
}

\begin{document}

\maketitle

\begin{abstract}
% Score matching estimators for point processes have gained widespread attention in recent years because they do not require the calculation of intensity integrals, thereby effectively addressing the computational challenges in maximum likelihood estimation (MLE). 
% Some existing works have proposed score matching estimators for point processes. 
% However, this work demonstrates that the incompleteness of the estimators proposed in those works renders them applicable only to specific problems, and they fail for more general point processes. 
% To address this issue, this work introduces the weighted score matching estimator to point processes. 
% Theoretically, we prove the consistency of our estimator and establish its rate of convergence. 
% Experimental results indicate that our estimator accurately estimates model parameters on synthetic data and yields results consistent with MLE on real data. In contrast, existing score matching estimators fail to perform effectively. 

Score matching estimators have gained widespread attention in recent years partly because they are free from calculating the integral of normalizing constant, thereby addressing the computational challenges in maximum likelihood estimation (MLE).  Some existing works have proposed score matching estimators for point processes. However, this work demonstrates that the incompleteness of the estimators proposed in those works renders them applicable only to specific problems, and they fail for more general point processes. To address this issue, this work introduces the weighted score matching estimator to point processes. Theoretically, we prove the consistency of our estimator and establish its rate of convergence. Experimental results indicate that our estimator accurately estimates model parameters on synthetic data and yields results consistent with MLE on real data. In contrast, existing score matching estimators fail to perform effectively. Codes are publicly available at \url{https://github.com/KenCao2007/WSM_TPP}. 
\end{abstract}

\section{Introduction}
\label{intro}
Point processes are a class of statistical models used to characterize event occurrences. Typical models include Poisson processes~\citep{kingman1992poisson} and Hawkes processes~\citep{hawkes1971spectra}. Their applications span various fields such as seismology~\citep{ogata1998space,ogata1999seismicity}, finance~\citep{bacry2014hawkes,hawkes2018hawkes}, criminology~\citep{mohler2013modeling}, and neuroscience~\citep{linderman2016bayesian,zhou2022efficient}. 
In the field of point processes, maximum likelihood estimation (MLE) has been a conventional estimator. 
However, MLE has an inherent limitation: it requires the computation of the normalizing constant, which corresponds to the intensity integral term in the likelihood. 
Except for simple cases, calculating the intensity integral analytically is generally infeasible. 
This necessitates the use of numerical integration methods like Monte Carlo or quadrature for approximating the computation. 
This introduces approximation errors, and more importantly, for high-dimensional problems, numerical integration encounters the curse of dimensionality, rendering training infeasible. 
% These shortcomings make MLE not an optimal choice for point processes in certain situations. 

To address this issue, prior research has introduced the concept of score matching (SM)~\citep{Hyvarinen05} to the field of point processes. For instance, 
\cite{SahaniBM16} derived the application of SM to the estimation of traditional statistical Poisson processes. Furthermore, \cite{zhang2023integration} extended the use of SM to the estimation of deep covariate spatio-temporal point processes. \cite{li2023smurf} also generalized the application of SM to Hawkes processes. 
These works have greatly advanced the utilization of SM for point processes. 
However, in practical applications, we have found that these estimators only work for specific point processes. For more general cases, these estimators cannot accurately estimate model parameters, even for some simple statistical point processes. One of the core contribution of this work is to \emph{theoretically} demonstrate the incompleteness of the estimators proposed in the aforementioned studies. 

The incompleteness of the estimators in the aforementioned studies stems from the transition from explicit SM to implicit SM. 
The explicit SM estimates model parameters by minimizing the expected distance between the gradient of the log-density of the model and the gradient of the log-density of the data. 
However, we cannot directly minimize the above objective function since it depends on the unknown data distribution. To facilitate solving, we need to convert the above explicit SM to implicit SM by using a trick of integration by parts, provided that some regularity conditions are satisfied~\citep{Hyvarinen05}. 
In \cite{SahaniBM16, zhang2023integration, li2023smurf}, they assume that the required regularity conditions are satisfied in their point process models and directly employ the implicit SM objective. However, as demonstrated in \cref{method}, the required regularity conditions cannot be met for general point processes. This implies that the concise implicit SM objectives (Equation (2) in \cite{SahaniBM16}, Equation (10) in \cite{zhang2023integration}, Equation (4) in \cite{li2023smurf}) are incomplete, and they cannot accurately estimate parameters for general point processes. 

To address this issue, 
% inspired by \cite{hyvarinen2007some,yu2019generalized}, 
this work introduces a (autoregressive) weighted score matching (WSM) estimator that can be applied to more general point processes. WSM eliminates the intractable terms in SM objective by adding a weight function that takes zero at the boundary of the integration region. 
% Unlike SM, WSM can be used for random variables whose support is a connected open set in $\mathbb{R}^d$. 
% as shown in \cite{liu2022estimating}. 
Compared to previous work on WSM \cite{hyvarinen2007some, yu2019generalized, liu2022estimating}, we are the first work to apply WSM on a stochastic process where the dimension $N_T$ is also random. This stochasticity in dimensionality poses greater challenges to the derivation, requiring special treatment to address this issue. 
% As far as we know, no previous literature has applied this kind of method to a stochastic point process. 
% whose domain is more complicated than that of a random variable. 

Specifically, we make following contributions: 
\textbf{(1)} We theoretically demonstrate that implicit (autoregressive) SM estimators in \cite{SahaniBM16, zhang2023integration, li2023smurf} are incomplete because the required regularity conditions cannot be satisfied for general point processes. 
\textbf{(2)} To address this issue, we propose a (autoregressive) WSM estimator that is applicable to general point processes. Theoretically, we establish its consistency and convergence rate. 
\textbf{(3)} In experiments, we confirm that on synthetic data, (autoregressive) WSM successfully recovers the ground-truth parameters; on real data, (autoregressive) WSM estimates results consistent with MLE; while existing (autoregressive) SM estimator fails in both scenarios. 
% Lastly, we conduct ablation studies to explore the robustness of our method concerning various components or variables. 

% The structure of this paper is as follows: \cref{preliminary} provides fundamental knowledge about Poisson processes, Hawkes processes, score matching, and autoregressive score matching. \cref{method} explains why implicit SM fails for Poisson processes and how WSM can address this issue. \cref{method2} discusses the similar problem for Hawkes processes. 
% In \cref{consistency}, we theoretically prove the consistency of (autoregressive) WSM. 
% In \cref{experiment}, we conduct experiments on both synthetic and real data. 
% \cref{limitation} discusses some limitations of this study and outlines future research directions. 
% Finally, \cref{conclusion} concludes this work. 

\section{Preliminaries}
\label{preliminary}
Now we provide knowledge on Poisson and Hawkes processes and (autoregressive) score matching.

\subsection{Poisson Process and Hawkes Process}\label{section:define poisson and hawkes process}

The Poisson process~\cite{kingman1992poisson} is a stochastic point process that models the occurrence of events over a time window $[0,T]$. A trajectory from Poisson process can be represented as an ordered sequence $\mathcal T = (t_1,\ldots, t_{N_T})$ where $N_t = \max\{n: t_n\leq t, t\in[0,T]\}$ is the corresponding counting process and thus $N_T$ is the random number of events in $[0, T]$.
The inhomogeneous Poisson process has a time-varying intensity $\lambda(t)$ representing the instantaneous rate of event occurrence at $t$. 
Mathematically, the intensity function is defined as $\lambda(t)=\lim_{\delta_t\to 0}\mathbb{E}[N_{t+\delta_t}-N_t]/\delta_t$. 
% If the intensity function is a constant, we obtain the homogeneous Poisson process. 
The probability density function of Poisson process is: 
\begin{equation}
    p(\mathcal T)=\prod_{n=1}^{N_T}\lambda(t_n)\exp\left(-\int_0^T \lambda(t)dt\right). 
\end{equation}

The Hawkes process~\cite{hawkes1971spectra} is a self-excitation point process where the occurrence of an event increases the likelihood of more events in the future. 
A trajectory from Hawkes process is similarly represented as $\mathcal T = (t_1,\ldots, t_{N_T})$ on $[0,T]$. The conditional intensity function of Hawkes process, representing the instantaneous rate of event occurrence at $t$ given the history up to but not including $t$, is: 
\begin{equation}
    \lambda^*(t)=\lambda(t|\mathcal{F}_{t^-})= \mu(t) + \sum_{t_j<t}g(t-t_j), 
    \label{eq:Hawkes intensity}
\end{equation}
where $\mu(t)$ is the baseline intensity, $g(\cdot)$ is the triggering kernel representing the self-excitation effect, the summation expresses the accumulative excitation from all past events, $\mathcal{F}_{t^-}$ is the historical information up to but not including $t$. 
$\lambda^*(t)$ means the intensity is dependent on the history. 
% Let $\{t_i\}_{i=1,2,\ldots}$ be a simple point process on the fixed interval $[0,T]$, i.e., $0 < t_1 < t_2 < \ldots \leq T$. Let the associated counting process be $N_T = \max\{i: t_i\leq t, t\in[0,T]\}$. 
% Let $(\Omega, \mathcal A, \mathbb P)$ be the underlying probability space upon which all the relevant random elements are defined. 
% Denote the history of $N_T$ by $\mathcal F = \{\mathcal F_t\}_{t\in [0,T]}$, so that $\mathcal F_t = \sigma\{N(s); s\leq t\}$. 
% The intensity function of $N_T$ is $\mathcal F$ predictable and takes the form as: 
% \begin{equation}
%     \lambda^*(t) = \mu(t)+\int_{[0,t)}g(t-u)dN(u), 
%     \label{eq:Hawkes intensity}
% \end{equation}
% where $*$ indicates the intensity depends on the history. 
% Such a process $N_T$ is defined as a Hawkes process in the paper. 

Poisson process assumes the independence of event occurrences, while Hawkes process extends it by introducing an autoregressive structure, making subsequent events dependent on prior events. 
Given the history $\mathcal{F}_{t_n}=(t_1,\ldots,t_n)$, the conditional probability density of $(n+1)$-th event at $t>t_n$ is: 
\begin{equation}
 p(t|\mathcal{F}_{t_n}) = \lambda^*(t)\exp\left(-\int_{t_n}^t \lambda^*(\tau)d\tau\right). 
    % \begin{aligned}
    %         p(t|\mathcal{F}_u) &= \lim_{dt\rightarrow 0} \frac{\text{Pr}(t_{N(u)+1}\in [t, t+dt]\big|\mathcal{F}_u)}{dt}\\&=
    %         \lambda(t) exp(-\int_{t_{N(u)}}^\top \lambda(\tau)d\tau.
    % \end{aligned}
    \label{eq:conditional pdf}
\end{equation}
% Frequently, we consider the history up to $i$-th timestamp, then $\mathcal F_{t_n} = \sigma(t_1,\ldots, t_n)$. We sometimes write $p(t_{i+1}| t_1,\ldots, t_{i})$ standing for $p(t_{i+1}| \mathcal F_{t_n})$. 
Here, we introduce the definition of the univariate Hawkes process. However, multivariate Hawkes processes also exist. For ease of notation, we use the univariate case for illustration, but we also provide solutions for the multivariate case.

\subsection{Score Matching}
MLE is a classic estimator that minimizes Kullback–Leibler divergence between a model distribution and data distribution. However, a drawback is the intractable computation of the normalizing constant. Approximating it through numerical integration can be computationally demanding. 
% particularly with higher input dimensions. 
In contrast, SM~\citep{Hyvarinen05} offers an alternative by minimizing Fisher divergence between model and data distributions: 
\begin{equation*}
    \mathcal{L}_{\text{SM}}(\theta)=\frac{1}{2}\mathbb{E}_{p(\mathbf{x})}\Vert\nabla_{\mathbf{x}}\log p(\mathbf{x})-\nabla_{\mathbf{x}}\log p_{\theta}(\mathbf{x})\Vert^2, 
    \label{esm}
\end{equation*}
where $p(\mathbf{x})$ represents the data distribution, $p_{\theta}(\mathbf{x})$ is the parameterized model distribution, the gradient of the log-density is called the score, and $\Vert\cdot\Vert$ represents a suitable norm, such as the $\ell^2$ norm. Minimizing the Fisher divergence above provides the parameter estimate. 
The advantage of SM lies in its ability to bypass the computation of the normalizing constant since the score no longer contains this constant: $p_{\theta}(\mathbf{x})=\frac{1}{Z(\theta)}\tilde{p}_{\theta}(\mathbf{x})$ where $Z(\theta)=\int \tilde{p}_{\theta}(\mathbf{x})d\mathbf{x}$, $\nabla_{\mathbf{x}}\log p_{\theta}(\mathbf{x})=\nabla_{\mathbf{x}}\log\tilde{p}_{\theta}(\mathbf{x})$. 

Under certain conditions, we can use integration by parts to replace the explicit SM objective, which involves an unknown distribution $p(\mathbf{x})$, with an equivalent implicit one, 
\begin{equation}
\label{eq:ISM objective}
\mathcal{J}_{\text{SM}}(\theta)=\mathbb{E}_{p(\mathbf{x})}\left[\frac{1}{2}\Vert\nabla_{\mathbf{x}}\log p_{\theta}(\mathbf{x})\Vert^2+\text{Tr}\left(\nabla^2_{\mathbf{x}}\log p_{\theta}(\mathbf{x})\right)\right]. 
\end{equation}
% where $\text{Tr}(\cdot)$ means the trace of a matrix. 
% In this form, it becomes a practical and usable objective function. 

\subsection{Autoregressive Score Matching}
% The original score matching method encounters efficiency challenges due to the computational cost associated with computing $\text{Tr}\left(\nabla^2_{\mathbf{x}}\log p_{\theta}(\mathbf{x})\right)$, especially when dealing with high-dimensional data.
% To address these efficiency concerns, several alternative methods have been proposed. These include denoising score matching \citep{Vincent11}, sliced score matching \citep{song2020sliced}, and the focus of our discussion, autoregressive score matching (ASM) \citep{meng2020autoregressive}. 
An autoregressive model defines a probability density $p(\mathbf{x})$ as a product of conditionals using the chain rule: $p(\mathbf{x})=\prod_{n=1}^N p(x_n|\mathbf{x}_{<n})$, 
    % \ \ \ \
    % p_{\theta}(\mathbf{x})=\prod_{n=1}^N p_{\theta}(x_n|\mathbf{x}_{<n}), 
% \end{equation*}
where $x_n$ is the $n$-th entry and $\mathbf{x}_{<n}$ denotes the entries with indices smaller than $n$. 
The original SM is not suitable for autoregressive models because the autoregressive structure introduces challenges in gradient computation in \cref{eq:ISM objective}. 
To address this issue, \cite{meng2020autoregressive} proposed autoregressive score matching (ASM). Unlike SM, which minimizes the Fisher divergence between the joint distributions of the model $p_{\theta}(\mathbf{x})$ and the data $p(\mathbf{x})$, ASM minimizes the Fisher divergence between the conditionals of the model $p_{\theta}(x_n|\mathbf{x}_{<n})$ and the data $p(x_n|\mathbf{x}_{<n})$: 
% The central idea behind ASM is to factorize the joint data (model) distribution into conditional distributions in an autoregressive fashion: 
% \begin{equation}
%     p(\mathbf{x})=\prod_{n=1}^N p(x_n|\mathbf{x}_{<n}) \ \ \ \
%     p_{\theta}(\mathbf{x})=\prod_{n=1}^N p_{\theta}(x_n|\mathbf{x}_{<n}), 
% \end{equation}
% where $x_n$ represents the $n$-th entry of $\mathbf{x}$, and $\mathbf{x}{<n}$ denotes the entries with indices smaller than $n$. 
% Subsequently, instead of directly matching joint distributions, ASM focuses on matching the conditionals of the model $p_{\theta}(x_n|\mathbf{x}{<n})$ to those of the data $p(x_n|\mathbf{x}{<n})$ using the one-dimensional score matching: 
\begin{equation*}
    \mathcal{L}_{\text{ASM}}(\theta)=\frac{1}{2}\sum_{n=1}^N\mathbb{E}_{p(\mathbf{x}_{\leq n})}\left(\frac{\partial\log p(x_n|\mathbf{x}_{<n})}{\partial x_n}-\frac{\partial\log p_{\theta}(x_n|\mathbf{x}_{<n})}{\partial x_n}\right)^2. 
\end{equation*}
% As demonstrated by \cite{meng2020autoregressive}, inspired by composite scoring rules~\citep{dawid2014theory}, this objective serves as a consistent estimator. 
Similarly, the above explicit ASM objective involves an unknown distribution $p(x_n|\mathbf{x}_{<n})$. Under specific regularity conditions, we can apply integration by parts to derive an implicit ASM objective: 
\begin{equation}
    \mathcal{J}_{\text{ASM}}(\theta)=\sum_{n=1}^N \mathbb{E}_{p(\mathbf{x}_{\leq n})}\left[\frac{1}{2}\left(\frac{\partial\log p_{\theta}(x_n|\mathbf{x}_{<n})}{\partial x_n}\right)^2+\frac{\partial^2 \log p_{\theta}(x_n|\mathbf{x}_{<n})}{\partial x_n^2}\right]. 
\label{asm}
\end{equation}

\section{Score Matching for Poisson Process}
\label{method}
We analyze the application of SM for Poisson process and its failure in achieving consistent estimation. Subsequently, we propose a provably consistent WSM estimator.

\subsection{Failure of Score Matching for Poisson Process}
% 需要满足的 regularity conditions 显示的写出来，并且说明 其实大部分的点过程都是不满足 这些 regularity conditions 比如 homo poisson 就不满足。
% 之后可以给出一个简单的例子，比如homo poisson 这种极为简单和常见的例子，从解析式上就可以看出sm失效了。
Consider a Poisson process $\mathcal{T} = (t_1, \ldots, t_{N_T})$ on $[0, T]$. 
% A single trajectory can be represented as a counting process $N_T$. Let $\mathcal T = (t_1,\ldots, t_{N_T})$ denotes the ordered timestamps. 
Let $p(\mathcal{T})$ represent the data distribution, which is uniquely associated with an intensity function $\lambda(t)$. 
Let $p_{\theta}(\mathcal{T})$ represent the parameterized model distribution, which is uniquely associated with a parameterized intensity function $\lambda_{\theta}(t)$. 
% We will use $\psi(t_n)$ to denote the $n$-th element in score function $\psi(t_n) = \frac{\partial \log p(\mathcal T)}{\partial t_n}$. 
In the following, we denote the score as $\psi(t_n)=\frac{\partial}{\partial t_n}\log p(\mathcal T)$. 

Previous works \cite{SahaniBM16,zhang2023integration} have both attempted to apply SM to the Poisson process: 
\begin{equation}
    \mathcal{L}_{\text{SM}}(\theta)=\frac{1}{2}\mathbb E_{p(\mathcal T)}\left[\sum_{n=1}^{N_T}\left(\psi(t_n)-\psi_{\theta}(t_n)\right)^2\right]. 
    \label{poisson_l_sm}
\end{equation}
In order for SM to be practical, \cite{SahaniBM16,zhang2023integration} assumed that specific regularity conditions are satisfied. Therefore, they employed an implicit SM objective similar to \cref{eq:ISM objective}: 
\begin{equation}
    \mathcal{J}_{\text{SM}}(\theta)=\mathbb{E}_{p(\mathcal{T})}\left[\sum_{n=1}^{N_T}\frac{1}{2}\psi^2_{\theta}(t_n)+\frac{\partial \psi(t_n)}{\partial t_n}\right]. 
\label{poisson_j_sm}
\end{equation}

In practical applications, we have found that the above estimator works only for specific Poisson processes and fails for more general Poisson processes. 
The reason for its failure lies in the fact that, for more general Poisson processes, the specific regularity conditions cannot be satisfied. 

Such conditions require the probability density function of the random variable is zero when it approaches infinity in any of its dimensions. However, for point processes, such requirement is not satisfied, because the random variable in point process $\mathcal T = (t_1,\ldots, t_{N_T})$ is not of fixed dimension and takes values in a subset of $\mathbb{R}_+^{N_T}$. 
% Such conditions require the domain of the random variable to be in $\mathbb R^N$, since it requires the pdf of $X$ is zero when $X$ approaches infinity in any of its dimensions. 
% However, for the Poisson or other Point Process, such requirements usually fail. The Point Process $\mathcal T$ is not a fixed dimensional random variable and takes values in a subset of $\mathbb{R}_+^{N_T}$. 
Therefore, for general Poisson processes, we cannot derive the implicit SM in \cref{poisson_j_sm} based on the explicit SM in \cref{poisson_l_sm}. 
% Therefore, for Temporal Point Process, ISM and ESM usually are not equivalent optimization objectives. For Poisson Process, we can summarise this observation as follows,

\begin{proposition}\label{prop:ESM and ISM for Poisson Process}
    Assume that all functions and expectations in $\mathcal L_{\text{SM}}(\theta)$ and $\mathcal J_{\text{SM}}(\theta)$ are well defined, we have, 
    \begin{equation}\label{eq:nuisance term in sm}
    \begin{aligned}
        \mathcal L_{\text{SM}}(\theta)=&\mathcal J_{\text{SM}}(\theta)+\text{const}-\sum_{N=1}^\infty \int p(t_1,\ldots, t_N)\frac{\partial \log p_{\theta}(t_1,\ldots,t_N)}{\partial t_1}\Big|_{t_1=0}d\mathcal T_{2:N}\\
        &+\sum_{N=1}^\infty \int p(t_1,\ldots, t_N)\frac{\partial \log p_{\theta}(t_1,\ldots,t_N)}{\partial t_N}\Big|_{t_N=T}d\mathcal T_{1:N-1}.
    \end{aligned}
    \end{equation}
    Therefore, $\mathcal L_{\text{SM}}(\theta)$ is equivalent to $\mathcal J_{\text{SM}}(\theta)$ if and only if the sum of the last two terms is a constant not containing $\theta$. 
\end{proposition}

% The sufficient condition for $\mathcal J_{\text{SM}}$ to be equivalent to $\mathcal L_{\text{SM}}$ is $p(t_1,\ldots, t_n){\lvert}_{t_1=0} = p(t_1,\ldots, t_n){\lvert}_{t_n=T} = 0$ 
For specific Poisson processes, the sum of the last two terms can be zero. However, for more general cases, this sum contains $\theta$. This implies that $\mathcal J_{\text{SM}}$ fails for general Poisson processes. 
% We demonstrate this with an example below. 
% \begin{example}
%     Consider an inhomogeneous Poisson process with intensity $\lambda(t)=\rho t^{\rho-1}$ with $\rho$ being the ground-truth parameter. 
%     The observation is $\{t_1^{(m)},\ldots, t_{N_m}^{(m)}\}_{m=1}^M$ where $M$ is the number of sequences. 
%     We assume the model intensity is $\lambda_{\theta}(t)=\theta t^{\theta - 1}$, so we can compute the score as: $\psi_{\theta}(t_n) = \frac{\theta - 1}{t_n}, \ \ \ \frac{\partial \psi_{\theta}(t_n)}{\partial t_n} = \frac{1- \theta}{t_n^2}$. 
%     % \begin{equation*}
%     %     \psi_{\theta}(t_n) = \frac{\theta - 1}{t_n}, \ \ \ \frac{\partial \psi_{\theta}(t_n)}{\partial t_n} = \frac{1- \theta}{t_n^2}, 
%     % \end{equation*}
%     Substituting the above into \cref{poisson_j_sm}, we obtain: 
%     \begin{equation*}
%         \hat{\mathcal J}_{\text{SM}}(\theta) = \frac{1}{2}(\theta-1)(\theta-3)\sum_{m=1}^M\sum_{n=1}^{N_m}\frac{1}{(t_n^{(m)})^2}. 
%     \end{equation*}
%     Minimizing such an objective results in $\hat \theta = 2$. This estimate is even independent of the observations, which is clearly an erroneous estimate. 
%     % This implies that, regardless of the observation, the estimated intensity function is always the same. 
%     In fact, in this example, the sum of the last two terms in \cref{eq:nuisance term in sm} is not zero and contains the parameter $\theta$, which explains the failure of SM. 
% \end{example}

\subsection{Weighted Score Matching}

To address the situation where SM fails, inspired by \cite{hyvarinen2007some,yu2019generalized}, we introduce the WSM for Poisson process. The core idea of WSM is to eliminate the two intractable terms by adding a weight function that takes zero at the boundary of the integration region. 
The weight function is designed to be a vector-valued function $\mathbf{h}: \mathbb R^{N_T}_+  \rightarrow \mathbb{R}_+^{N_T}$ with the $n$-th element denoted as $h_n(\mathcal T)$. 
% In WSM, the most crucial step is to design a suitable weight function for the distribution to be estimated. 
Here, we present the conditions that a valid weight function should satisfy: 
\begin{equation}\label{eq:wsm_condition}
\begin{gathered}
    \lim_{t_n\rightarrow t_{n+1}}p(\mathcal T)\psi_{\theta}(t_n)h_n(\mathcal T)=0,\ \lim_{t_n\rightarrow t_{n-1}}p(\mathcal T)\psi_{\theta}(t_n)h_n(\mathcal T)=0,\ \forall n\in [N_T],\\
    \mathbb E[\psi^2_{\theta}(t_n)h_n(\mathcal T)]<\infty,\ \mathbb E[\psi_{\theta}(t_n)\frac{\partial h_n(\mathcal T)}{\partial t_n}]<\infty, \forall n\in [N_T]. 
\end{gathered}
\end{equation}
One can verify that such weight functions are easy to find for most $p(\mathcal T)$ and $\psi_\theta(t_n)$. With a valid weight function $\mathbf{h}$, the explicit WSM objective can be defined as: 
\begin{equation}
\begin{aligned}
        &\mathcal L_{\text{WSM}}(\theta) = \frac{1}{2}\mathbb E_{p(\mathcal T)}\left[\sum_{n=1}^{N_T}(\psi(t_n)-\psi_{\theta}(t_n))^2 h_n(\mathcal T)\right]. 
\end{aligned}
\end{equation}
The introduction of the weight function allows control over the values of the integrand at the boundaries of the integration domain, thereby eliminating the last two terms in \cref{eq:nuisance term in sm}. 
\begin{theorem}\label{thm:validity of WSM objective}
    Assume the true intensity is in the family of the model intensity, denoted as $\lambda(t) = \lambda_{\theta^*}(t)$, where $\theta^* \in \Theta$.  We further assume that $\frac{\partial \log\lambda_{\theta_1}(t)}{\partial t}=\frac{\partial \log\lambda_{\theta_2}(t)}{\partial t}\ a.s.$ gives $\theta_1 = \theta_2$. Then the unique minimizer of $\mathcal L_{\text{WSM}}(\theta)$ is $\theta^*$.
\end{theorem}
The explicit WSM objective is not practical as it depends on the unknown data distribution $p(\mathcal{T})$, so we further derive the implicit WSM objective which is tractable. 
\begin{theorem}\label{thm:equivalence between EWSM and IWSM}
    Assume that all functions and expectations in $\mathcal L_{\text{WSM}}(\theta)$ and  $\mathcal J_{\text{WSM}}(\theta)$ are well defined, \cref{eq:wsm_condition} is satisfied, we have, 
    \begin{equation*}
        \mathcal L_{\text{WSM}}(\theta)=\mathcal J_{\text{WSM}}(\theta) + \text{const},
    \end{equation*}
    % where 
    \begin{equation}
    \begin{aligned}
                &\mathcal J_{\text{WSM}}(\theta)=\mathbb E_{p(\mathcal T)}\left[\sum_{n=1}^{N_T} \frac{1}{2}\psi^2_{\theta}(t_n)h_n(\mathcal T) + \frac{\partial \psi_{\theta}(t_n)}{\partial t_n}h_n(\mathcal{T})+  \psi_{\theta}(t_n)\frac{\partial h_n(\mathcal T)}{\partial t_n} \right]. 
    \end{aligned}
    \label{poisson_j_wsm}
    \end{equation}
\end{theorem}
For general Poisson processes, \cref{poisson_j_wsm} is always valid with a suitable weight function. Thus, we do not need to worry about the issues of failure that may arise when using \cref{poisson_j_sm}. 

\section{Autoregressive Score Matching for Hawkes Processes}
\label{method2}
Similarly, we analyze the usage of ASM for Hawkes processes and its failure in achieving consistent estimation. 
Subsequently, we propose a provably consistent autoregressive WSM (AWSM) estimator. 

\subsection{Failure of Autoregressive Score Matching for Hawkes Process}
The original SM, even when adjusted by a weight function, is not suitable for point processes with autoregressive structures, such as Hawkes process. Because in such cases, directly calculating the score still includes the intensity integral, which is precisely what the use of SM aims to avoid. 
Therefore, an ASM method is proposed for parameter estimation for Hawkes process in \cite{li2023smurf}. 

Consider a Hawkes process $\mathcal{T} = (t_1, \ldots, t_{N_T})$ on $[0, T]$ with the underlying conditional probability density of $t_n$ denoted as $p(t_n|\mathcal F_{t_{n-1}})$. 
The parameterized conditional probability density model of $t_n$ is $p_{\theta}(t_n|\mathcal F_{t_{n-1}})$. 
We denote the conditional score as $\psi(t_n|\mathcal F_{t_{n-1}})=\frac{\partial}{\partial t_n}\log p(t_n|\mathcal F_{t_{n-1}})=\frac{\partial}{\partial t_n}\log \lambda(t_n|\mathcal F_{t_{n-1}}) - \lambda(t_n|\mathcal F_{t_{n-1}}), n = 1,\ldots N_T$. 
An explicit ASM objective is defined as: 
\begin{equation}\label{eq:smurf loss}
\begin{aligned}
        &\mathcal L_{\text{ASM}}(\theta) =\frac{1}{2}\mathbb E_{p(\mathcal{T})}\left[\sum_{n=1}^{N_T}(\psi(t_n|\mathcal F_{t_{n-1}}) - \psi_\theta(t_n|\mathcal F_{t_{n-1}}))^2\right].
\end{aligned}
\end{equation}
Similarly, to make ASM practical, \cite{li2023smurf} assumed that specific regularity conditions are satisfied. Therefore, an implicit ASM is proposed accordingly: 
\begin{equation}
\begin{aligned}
       &\mathcal J_{\text{ASM}}(\theta) = \mathbb E_{p(\mathcal{T})}\left[\sum_{n=1}^{N_T}\frac{1}{2}\psi^2_{\theta}(t_n|\mathcal F_{t_{n-1}})+\frac{\partial \psi_{\theta}(t_n|\mathcal{F}_{t_{n-1}})}{\partial t_n}\right]. 
\end{aligned}
\label{J_asm}
\end{equation}

However, the same issue as in the Poisson process arises here. The regularity conditions required to eliminate the unknown data distribution do not hold. 
Therefore, we cannot derive the implicit ASM in \cref{J_asm} based on the explicit ASM in \cref{eq:smurf loss}. 

\begin{proposition}\label{prop:equivalency}
    Assume that all functions and expectations in $\mathcal L_{\text{ASM}}(\theta)$ and  $\mathcal J_{\text{ASM}}(\theta)$ are well defined,  we have,
    \begin{equation}\label{eq:nuisance term}
        \begin{aligned}
            \mathcal L_{\text{ASM}}(\theta) = &\mathcal J_{\text{ASM}}(\theta) + \text{const} 
            +\sum_{n=1}^{\infty} \int p(\mathcal T_{:n-1})p(t_n|\mathcal F_{t_{n-1}})\psi_{\theta}(t_n|\mathcal F_{t_{n-1}})\Big|_{t_n=T}d\mathcal{T}_{:{n-1}}\\
            -&\sum_{n=1}^{\infty} \int p(\mathcal T_{:n-1})p(t_n|\mathcal F_{t_{n-1}})\psi_{\theta}(t_n|\mathcal F_{t_{n-1}})\Big|_{t_n=t_{n-1}}d\mathcal{T}_{:{n-1}}. 
        \end{aligned}
    \end{equation}
    Therefore, $\mathcal L_{\text{ASM}}(\theta)$ is equivalent to $\mathcal J_{\text{ASM}}(\theta)$ if and only if the sum of last two terms is a constant not containing $\theta$.
\end{proposition}

% The proof is provided in \cref{seq:proof of prop equivalency}

Generally speaking, for most Hawkes processes, the sum of the last two terms in \cref{eq:nuisance term} still contains $\theta$, even for a common Hawkes process with an exponential decay triggering kernel. We illustrate this example in \cref{synthetic}. This implies that $\mathcal J_{\text{ASM}}$ fails for general Hawkes processes.

% the intractable integration term in \cref{eq:nuisance term} doesn't cancel out for most of the ground truth Hawkes process, as illustrated in the experiment part(!! experiment). 
% Usually we would expect $p(t_n|\mathcal F_{t_{n-1}})=0$ at $t_n = t_{n-1}$ and $t_n = T$. However, for non-trivial Hawkes process, we always have,
% \begin{equation*}
%     \begin{aligned}
%         &p(t_n|\mathcal F_{t_{n-1}})\big|_{t_{n-1}}\\&=\mu(t_{n-1})+g(0) + \sum_{t_j<t_{n-1}}g(t_{n-1}-t_j)\\
%         &>\lambda(t_{n-1}|\mathcal F_{t_{n-2}})>0
%     \end{aligned}
% \end{equation*}
% The last inequality is because 
% This implies a sufficient condition for $\mathcal J_{ASM}$ being equivalent to $\mathcal L_{ASM}$ is difficult to find. 

\subsection{Autoregressive Weighted Score Matching}
% Now we introduce the autoregressive weighted score matching objective to estimate a Hawkes Process.
Similarly, to address the situation where ASM fails, we introduce the AWSM for Hawkes process. 
We present the conditions that a valid weight function $\mathbf{h}$ should satisfy : 
\begin{gather}
    \lim_{t_n\rightarrow T}p(\mathcal T_{1:n})\psi_{\theta}(t_n|\mathcal F_{t_{n-1}})h_n(\mathcal T) = 0, \lim_{t_n\rightarrow t_{n-1}}p(\mathcal T_{1:n})\psi_{\theta}(t_n|\mathcal F_{t_{n-1}})h_n(\mathcal T) = 0,\ \forall n \in [N_T], 
    % \label{eq:awsm_require1} 
    \nonumber\\
    \mathbb E[\psi_{\theta}^2(t_n|\mathcal F_{t_{n-1}})h_n(\mathcal T)]<\infty,\ \mathbb E[\psi_{\theta}(t_n|\mathcal F_{t_{n-1}})\frac{\partial h_n(\mathcal T)}{\partial t_n}]<\infty, \forall n\in [N_T].\label{eq:awsm_require2}
\end{gather}

With a valid weight function $\mathbf{h}$, the explicit AWSM objective can be defined as: 
\begin{equation}
\label{eq:Explicit Autoregressive Weighted Score Matching}
\begin{aligned}
        &\mathcal L_{\text{AWSM}}(\theta) =  \frac{1}{2}\mathbb E_{p(\mathcal{T})}\left[\sum_{n=1}^{N_T}(\psi(t_n|\mathcal F_{t_{n-1}}) -\psi_\theta(t_n|\mathcal F_{t_{n-1}}))^2h_n(\mathcal T)\right]. 
\end{aligned}
\end{equation}
% The introduction of the weight function allows control over the values of the integrand at the boundaries of the integration domain, thereby 
% eliminating the last two terms in \cref{eq:nuisance term}. 
% Before we derive an equivalent and tractable objective from $\mathcal L_{\text{AWSM}}(\theta)$, we first show the validity of \cref{eq:Explicit Weighted Score Matching} for parameter estimation.

\begin{theorem}\label{theorem:unique minimizer}
    Assume the true conditional density is in the family of the model conditional density, denoted as $p(t_n|\mathcal F_{t_{n-1}}) = p_{\theta^*}(t_n|\mathcal F_{t_{n-1}})$, where $\theta^* \in \Theta$. We further assume that $p_{\theta_1}(t_n|\mathcal F_{t_{n-1}}) = p_{\theta_2}(t_n|\mathcal F_{t_{n-1}})\ a.e.$ gives $\theta_1 = \theta_2$. Then the unique minimizer of $\mathcal L_{\text{AWSM}}(\theta)$ is  $\theta^*$. 
\end{theorem}

The explicit AWSM objective is not practical as it depends on the unknown data distribution $p(t_n|\mathcal F_{t_{n-1}})$, so we further derive the implicit AWSM objective which is tractable. 
% Now we provide the derivation of an implicit Autoregressive Weighted Score Matching objective which is tractable and equivalent to $\mathcal L_{\text{AWSM}}(\theta)$. 
\begin{theorem}
\label{thm:tractable EWSM}
% Assume that all functions and expectations in \cref{eq:awsm_require2} are satisfied, then we have, 
Assume that all functions and expectations in $\mathcal L_{\text{AWSM}}(\theta)$ and  $\mathcal J_{\text{AWSM}}(\theta)$ are well defined, \cref{eq:awsm_require2} are satisfied, we have, 
\begin{equation*}
    \begin{aligned}
        \mathcal L_{\text{AWSM}}(\theta) &= \mathcal J_{\text{AWSM}}(\theta) + \text{const}, 
    \end{aligned}
\end{equation*}
% where
\begin{equation}
\label{eq:implicit autoregressive weighted score matching}
\begin{aligned}
&\mathcal J_{\text{AWSM}}(\theta) = \mathbb E_{p(\mathcal{T})}\left[\sum_{n=1}^{N_T}\frac{1}{2}\psi^2_{\theta}(t_n|\mathcal F_{t_{n-1}})h_n(\mathcal T)
+\frac{\partial \psi_{\theta}(t_n|\mathcal F_{t_{n-1}})}{\partial t_n}h_n(\mathcal T) + \psi_{\theta}(t_n|\mathcal F_{t_{n-1}})\frac{\partial h_n(\mathcal T)}{\partial t_n}\right]. 
\end{aligned}
\end{equation}
\end{theorem}
% The proof is provided in the appendix.

For general Hawkes processes, \cref{eq:implicit autoregressive weighted score matching} is always valid with a suitable weight function. Thus, we do not need to worry about the issues of failure that may arise when using \cref{J_asm}.

\paragraph{Multivariate Hawkes Processes}
For the multivariate case, events are $\{(t_1, k_1), \ldots, (t_{N_T}, k_{N_T})\}$ with $k_n \in {1, \ldots, K}$ denoting the event type of the $n$-th event. 
The history up to the $(n-1)$-th event is denoted by $\mathcal F_{t_{n-1}}$. 
% Notably, $\mathcal F_{t_{n-1}}$ includes the type information of historical events.
We need to consider both the distributions of event times and event types. 
For the temporal distribution, we use the AWSM objective with the temporal score $\psi(t_n|\mathcal F_{t_{n-1}}) = \frac{\partial}{\partial t_n}\log p(t_n|\mathcal F_{t_{n-1}})$ as before.
For the type distribution, since we do not need to compute the intensity integral, we directly use the cross-entropy objective: 
\begin{equation}\label{eq:CE_loss}
    \mathcal{J}_{\text{CE}}(\theta) = \mathbb E_{p(\mathcal T)}\left[\sum_{n=1}^{N_T}\log p_{\theta}(k_n|\mathcal {F}_{t_{n-1}}, t_n)\right]=\mathbb E\left[\sum_{n=1}^{N_T} \log \lambda_{k_n}(t_n|\mathcal F_{t_{n-1}}; \theta) - \log \lambda(t_n|\mathcal F_{t_{n-1}}; \theta)\right], 
\end{equation}
where $\lambda=\sum_{k=1}^K \lambda_k$. 
The final loss is $\mathcal J(\theta) = \mathcal J_{\text{AWSM}}(\theta) + \alpha \mathcal{J}_{\text{CE}}(\theta)$; $\alpha$ is a balancing coefficient. 

 \section{Theoretical Analysis}
\label{sec:Theory}
In this section, we analyze the statistical properties of AWSM estimator of univariate Hawkes process. Similar conclusions also hold for the WSM estimator of Poisson process, as discussed in \cref{section:theory_poisson}. 
We consider $M$ i.i.d. sequences $\{t_1^{(m)}, \ldots, t_{N_m}^{(m)}\}_{m=1}^M$ from $p(\mathcal{T})$ of a Hawkes process. 
We assume the true density is in the family of the model density, denoted as $p(\mathcal T)=p_{\theta^*}(\mathcal T)$, where $\theta^* \in \Theta \subset \mathbb R^r$. 
% We also consider true density follows model $p(\mathcal T)=p(\mathcal T; \theta^*)$ for $\theta^* \in \Theta$.  
The estimate $\hat \theta$ is obtained by $\hat \theta = \argmin_{\theta \in \Theta}\hat {\mathcal J}_{\text{AWSM}}(\theta)$ where $\hat {\mathcal J}_{\text{AWSM}}$ represents the empirical loss. Below we omit the subscript AWSM as it does not cause any ambiguity. 
% The details of all the derivations are provided in \cref{ddd}. 

\subsection{Asymptotic Property}
\label{consistency}
% In this section, we prove that under a parametric setting, optimizing an (Autoregressive) Weighted Score Matching objective would result in a consistent estimator for Poisson and Hawkes process. 
% In the following, we consider collecting $M$ trajectories of realized Poisson or Hawkes Process denoted with $\{t_1^{(m)},\ldots, t_{N_m}^{(m)}\}_{m=1}^M$. 
% All trajectories follows the same true process density $p(\mathcal T)$ and are independent. 

We first establish the consistency of $\hat \theta$ for a Hawkes process. 
\begin{theorem}
\label{thm:AWSM_consistency}
    % Suppose Assumption \cref{assu:poisson_compact}, \cref{assu:hawkes_differentiable} and \cref{assu:hawkes_identifiable} hold and the ground-truth process is a Hawkes process with conditional intensity defined by $\mu_{\theta^*}(t)$ and $g_{\theta^*}(t)$, 
    Under mild regularity \cref{assu:poisson_compact,assu:hawkes_differentiable,assu:hawkes_identifiable}, we have $\hat \theta \xrightarrow{p} \theta^*$ as $M\rightarrow \infty$. 
\end{theorem}

\subsection{Non-asymptotic Error Bound}
Then, we establish a non-asymptotic error bound for $\hat \theta$. We define 
\begin{equation*}
    \begin{aligned}
    % A_n(\mathcal T,\theta) := \frac{1}{2}\psi^2_{\theta}(t_n|\mathcal F_{t_{n-1}}) + \frac{\partial \psi_{\theta}(t_n|\mathcal F_{t_{n-1}})}{\partial t_n}, B_n(\mathcal T,\theta) := \psi_{\theta}(t_n|\mathcal F_{t_{n-1}})
        \mathcal J_{\mathbf{h}}(\theta)&=\mathbb E_{p(\mathcal T)}\left[\sum_{n=1}^{N_T} \big[\underbrace{\frac{1}{2}\psi^2_{\theta}(t_n|\mathcal F_{t_{n-1}}) + \frac{\partial \psi_{\theta}(t_n|\mathcal F_{t_{n-1}})}{\partial t_n}}_{A_n(\mathcal T,\theta)}\big]h_n(\mathcal{T})+ \underbrace{\psi_{\theta}(t_n|\mathcal F_{t_{n-1}})}_{B_n(\mathcal T,\theta)} \frac{\partial h_n(\mathcal T)}{\partial t_n} \right].\\
    \end{aligned}
\end{equation*}

% First we assume that the optimal parameter $\theta^*$ is well-separated from other neighbouring parameters in terms of the population objective values:

\begin{assumption}\label{assump_seperate}
    Assume there exists $\alpha>1$ such that,
    \begin{equation*}
        \mathop{inf}_{\theta:||\theta-\theta^*||\geq \delta} \mathcal J_{\mathbf h}(\theta)-\mathcal J_{\mathbf{h}}(\theta^*)\geq C_{\mathbf{h}}\delta^{\alpha}
    \end{equation*}
    holds for any small $\delta$. Here, $C_{\mathbf{h}}$ is a positive constant that depends on the weight function $\mathbf h$ such that $C_{a\mathbf{h}}=aC_{\mathbf{h}}$ for any positive constant $a$. $\|\cdot\|$ is the euclidean norm.
\end{assumption}

% Second, we assume the continuity of objective function:

% \begin{restatable}{assumption}{assumpcontinuity}
\begin{assumption}
\label{assump_continuity}
    For $\forall n\in \mathbb N^+$, there exists $\dot{A_n}(\mathcal T), \dot{B_n}(\mathcal T)$ such that,
    \begin{equation*}
        \begin{aligned}
            |A_n(\mathcal T, \theta_1) - A_n(\mathcal T, \theta_2)| \leq \dot{A_n}(\mathcal T)||\theta_1-\theta_2||,\ \ \ \ |B_n(\mathcal T, \theta_1) - B_n(\mathcal T, \theta_2)| \leq \dot{B_n}(\mathcal T)||\theta_1-\theta_2||. 
        \end{aligned}
    \end{equation*}
\end{assumption}
% \end{restatable}

% First we assume that the optimal parameter $\theta^*$ is well-separated from other neighbouring parameters in terms of the population objective values:
% Second, we assume the continuity of objective function:
% Define,
% \begin{equation*}
%     \Gamma(\mathbf{h},A,B):= \mathbb E_{p(\mathcal T)}\left\{\sum_{n=1}^{N_T}\big[(\dot A_n^4(\mathcal T) h_n^4(\mathcal T))^{\frac{1}{4}} +  (\dot B_n^4(\mathcal T) h_n^4(\mathcal T))^{\frac{1}{4}}\big]\right\}
% \end{equation*}

\begin{theorem}\label{thm:error_bound}
    Given that $\hat \theta$ converges to $\theta^*$ in probability, combined with \cref{assump_seperate,assump_continuity}, for $\delta < CK_{\alpha}\frac{\sqrt {r}}{2^{\alpha-1}}\frac{\Gamma(\mathbf{h}, A, B)}{C_{\mathbf h}}$, we have 
\begin{equation}
    \text{Pr}\left[||\hat \theta - \theta^*|| \leq \left(CK_{\alpha}\frac{\Gamma(\mathbf{h}, A, B)}{\delta C_{\mathbf{h}}}\sqrt{\frac{r}{M}}^{1/(\alpha-1)}\right)\right]\geq 1-\delta,
    \label{upperbound}
\end{equation}
where $\Gamma(\mathbf{h},A,B)=\sqrt{\mathbb E_{p(\mathcal T)}\left\{\sum_{n=1}^{N_T}\big[(\dot A_n(\mathcal T) h_n(\mathcal T)) +  (\dot B_n(\mathcal T) \frac{\partial h_n(\mathcal T)}{\partial t_n})\big]\right\}^2}$, $C$ is a universal constant, $K_{\alpha}=\frac{2^{2\alpha}}{2^{\alpha-1}-1}$, and $r$ is the number of dimensions of ${\theta}$. 
\end{theorem}

\subsection{Discussion on Optimal Weight Function}\label{section:optimal_weight_hawkes}
In \cref{method,method2}, we only provide the conditions that the weight function needs to satisfy. In fact, there are many weight functions that satisfy these conditions. The optimal weight function should minimize the error bound in \cref{upperbound}, which is equivalent to minimizing the coefficient $\frac{\Gamma(\mathbf{h}, A, B)}{C_\mathbf{h}}$. 
The numerator cannot be analytically computed as it involves an unknown distribution $p(\mathcal{T})$, but we can maximize the denominator $C_{\mathbf{h}}$ in a predefined function family. 
% We cannot directly optimize this coefficient since the ground-truth density $p(\mathcal T)$ is unknown. However, we can directly maximize the denominator $C_{\mathbf{h}}$ in a defined function class.

\begin{theorem}\label{thm:optimal_denominator}
    Define $\mathbf{h}^0$ to be a weight function with its $n$-th element defined as the distance between $t_n$ and the boundary of its support $[t_{n-1}, T]$: 
\begin{equation*}
    h_n^0(t_n)=\frac{T-t_{n-1}}{2}-|t_n-(T+t_{n-1})/2|. 
\end{equation*}
    We have, 
     \begin{equation*}
        \bm h^0 \in \argmax_{\bm h \in \mathcal H} \mathop{\text{inf}}_{\theta:||\theta-\theta^*||\geq \delta} \mathcal J_{\mathbf{h}}(\theta)  - \mathcal J_{\mathbf{h}}(\theta^*)
        % \mathop{\text{inf}}_{\theta:||\theta-\theta^*||\geq \delta} \mathcal J_{\mathbf{h}^0}(\theta)  - \mathcal J_{\mathbf{h}^0}(\theta^*) \geq \mathop{\text{sup}}_{\mathbf{h}\in \mathcal H}  \mathop{\text{inf}}_{\theta:||\theta-\theta^*||\geq \delta} \mathcal J_{\mathbf{h}}(\theta)  - \mathcal J_{\mathbf{h}}(\theta^*), 
    \end{equation*}
    where $\mathcal H$ is a family of functions that is rigorously defined in \cref{eq:definition_of_H}. 
    % \begin{equation*}
    %     \mathcal H = \{\}
    % \end{equation*}
\end{theorem}
% Therefore, choosing $\mathbf{h}_n^0$ can optimize $C_{\mathbf{h}}$, which is the denominator of the coefficient $\frac{\Gamma(\mathbf{h}, A,B)}{C_{\mathbf{h}}}$. As to the numerator, it cannot be minimized explicitly. 
Combined with \cref{assump_seperate}, it can be observed that $\mathbf{h}^0$ maximizes $C_{\mathbf{h}}$ in $\mathcal H$. Though it does not necessarily optimize $\frac{\Gamma(\mathbf h, A, B)}{C_{\mathbf h}}$, it is an adequate choice without using any information on $p(\mathcal T)$. We also discuss it heuristically in \cref{sec:heuristic_disc_on_optimality}. 
It is worth noting that $h_n^0$ is not continuously differentiable; however, it is weakly differentiable. Its weak derivative is continuous, allowing both integration by parts and statistical theory to hold. In subsequent experiments, we consistently employ this optimal weight function when $T$ is available or can be approximated for the dataset.

\section{Experiments} 
\label{experiment}
In this section, we validate our proposed (A)WSM on parametric or deep point process models.
For parametric models, we focus on verifying whether (A)WSM can accurately recover the ground-truth parameters. For deep point process models, we confirm that our new training method is also applicable to deep neural network models. \footnote{Experiments are performed using an NVIDIA A16 GPU, 15GB memory.}

% In this section, we validate our proposed (A)WSM on synthetic and real point process data.
% For synthetic data, we focus on statistical point process models to verify whether (A)WSM can accurately recover the ground-truth parameters.
% For real data, we focus on recent deep point process models to confirm whether (A)WSM is applicable to the training of deep point process models.\footnote{Experiments are performed using a Quadro RTX 6000, 24GB memory.}

\subsection{Baselines and Metrics}
We consider three baseline parameter estimators: (1) \textbf{MLE} (2) implicit \textbf{(A)SM}~\citep{SahaniBM16, zhang2023integration, li2023smurf} (3)  Denoising Score Matching (\textbf{DSM})~\citep{li2023smurf}. We briefly introduce DSM in deep point process models.

For deep Hawkes process training, DSM is employed as follows. For observed timestamps $t_n^{(m)}$ in $m$-th sequence, we sample $L$ noise samples $\tilde t_{n,l}^{(m)}=t_n^{(m)}+ \epsilon_{n,l}^{(m)}, l = 1,\ldots, L,$ where $\text{Var}(\varepsilon_{n,L}^{(m)})=\sigma^2$ and get the DSM objective: 
\begin{equation*}
    \hat {\mathcal J}(\theta)=\frac{1}{M}\sum_{m=1}^M\sum_{n=1}^{N_m}\sum_{l=1}^{L}\frac{1}{2L}[\psi_{\theta}(\tilde t_{n,l}^{(m)}|\mathcal F_{t_{n-1}^{(m)}})+\frac{\varepsilon_{n,l}^{(m)}}{\sigma^2}]+\alpha\hat {\mathcal J}_{\text{CE}}(\theta),
\end{equation*}
where ${\mathcal J}_{\text{CE}}(\theta)$ is the cross-entropy loss defined in \cref{eq:CE_loss}.

% MLE can accurately estimate parameters when the intensity integral can be computed analytically or approximated well by numerical integration; 
% however, implicit (A)SM fails when the required regularity conditions do not hold. 

To compare the performance of different methods, for parametric models on synthetic data, we use the mean absolute error (\textbf{MAE}, $|\hat{\theta}-\theta|$) between the ground-truth parameters and the estimates as a metric since the ground-truth parameters are known. For deep point process models, we use the test log-likelihood (\textbf{TLL}) and the event type prediction accuracy (\textbf{ACC}) on the test data as metrics.

% when the ground-truth parameters are known (synthetic data), we use the mean absolute error (\textbf{MAE}, $|\hat{\theta}-\theta|$) between the ground-truth parameters and the estimates as a metric. When the ground-truth parameters are unknown (real data), we use the test log-likelihood (\textbf{TLL}) 
% and the event type prediction accuracy (\textbf{ACC}) on the test data as metrics. 

\subsection{Parametric Models}
% \subsection{Synthetic Data}

\label{synthetic}
\paragraph{Datasets} We validate the effectiveness of (A)WSM using three sets of synthetic data. 
% Each dataset is generated by using the thinning algorithm~\citep{ogata1998space}. 
(1) \textbf{Poisson Process}: This dataset is simulated from an inhomogeneous Poisson process with an intensity function $\lambda(t)=\exp(\theta \sin(t))$ with $T=2$, $\theta=2$. 
(2) \textbf{Exponential Hawkes Processes}: This dataset is simulated from $2$-variate Hawkes processes with exponential decay triggering kernels $g_{ij}(\tau) = \alpha_{ij} \exp(-5 \tau),\ \tau>0$ with $T=10$, $\mu_1=\mu_2=1$, $\alpha_{11}=1.6, \alpha_{12}=0.2$, $\alpha_{21}=\alpha_{22}=1$. 
    % \item \textbf{Exponential Decay Kernel} This kernel assumes that the influence of historical events decays exponentially as time elapses. The kernel function for an interval $\tau$ is given by: $\phi(\tau) = \alpha \exp(-\beta \tau)$, 
    % % \begin{equation*}
    % % \label{Eq: Exp Kernel Def}
    % % \phi(\tau) = \alpha \exp(-\beta \tau),
    % % \end{equation*}
    % where $\alpha = 1$ and $\beta = 2$. 
(3) \textbf{Gaussian Hawkes Processes}: This dataset is simulated from $2$-variate Hawkes processes with Gaussian decay triggering kernels $g_{ij}(\tau) =\frac{\alpha_{ij}}{\sqrt{2\pi}\sigma} \exp(-\frac{\tau^2}{2\sigma^2}),\ \tau>0$ with $T=10$, $\mu_1=\mu_2=1$, $\alpha_{11}=1.6, \alpha_{12}=0.2$, $\alpha_{21}=\alpha_{22}=1$, $\sigma=1$.

\paragraph{Training Protocol}
We assume that we know the ground-truth model but do not know its parameters. Therefore, we use the ground-truth model as the training model. 
The purpose is to verify whether the estimator can recover the ground-truth parameters. 
% For (A)SM and (A)WSM, we demonstrate that for multivariate Hawkes processes, the objective function may have multiple minimizers. 
For each dataset, we collect a total of $1000$ sequences. 
We 
% initialize the parameters at the ground truth position and 
run $500$ iterations of gradient descent using Adam~\cite{kingma2014adam} as the optimizer for all scenarios. 
For MLE, the intensity integral is computed through numerical integration, with the number of integration nodes set to $100$ to achieve a considerable level of accuracy. 
We change the random seed $3$ times to compute the mean and standard deviation of MAE.

\paragraph{Results}
In \cref{table: synthetic Data Experiment}, we report the MAE of parameter estimates for three models trained by MLE, (A)SM, and (A)WSM on the synthetic dataset. 
We can see that both MLE and (A)WSM achieve small MAE on three types of data. However, the MAE of (A)SM is large. As we have theoretically demonstrated earlier, this is because MLE and (A)WSM estimators are consistent. In contrast, (A)SM, due to the absence of the required regularity conditions in the three cases, has an incomplete objective and cannot accurately estimate parameters. 
In \cref{learned_intensity}, we showcase the learned intensity functions. Both MLE and (A)WSM successfully captured the ground truth, while (A)SM fails. 

% Furthermore, we present the convergence curves of (A)SM and (A)WSM during the model training process in \cref{fig: convergence}. It can be observed that (A)WSM exhibits a relatively stable converging pattern, eventually reaching a plateau. On the contrary, (A)SM shows instability with frequent fluctuations during the training process. This also confirms the superiority of (A)WSM over (A)SM, because (A)WSM is a consistent estimator, while (A)SM is incomplete. 

\begin{table*}[t]
\centering
\caption{The MAE of three models trained by MLE, (A)SM, and (A)WSM on the synthetic dataset. 
For the 2-variate processes, we only present the estimation results for some parameters here. The results for other parameters can be found in \cref{table:additional_experiments}.}

\label{table: synthetic Data Experiment}
\begin{sc}
\scalebox{0.74}{
\begin{tabular}{c|c|ccc|ccc}
    \toprule
    \multirow{2}{*}{Estimator} & \multicolumn{1}{c|}{Poisson} & \multicolumn{3}{c|}{Exp-Hawkes }  & \multicolumn{3}{c}{Gaussian-Hawkes } \\
    \cmidrule{2-8}
        & $\phi$ & $\alpha_{11}$ & $\alpha_{12}$ &  $\mu_{1}$ & $\alpha_{11}$ & $\mu_{1}$ & $\sigma$ \\
    
    \midrule
    (A)WSM & $0.07_{\pm 0.14}$ & $0.041_{\pm 0.041}$& {$0.026_{\pm 0.001}$ }&$\bm{0.011_{\pm 0.010}}$ &$0.153_{\pm 0.162}$  & {$0.022_{\pm 0.023}$} &$0.060_{\pm 0.066}$ 
    \\
     \midrule
    (A)SM & $1.56_{\pm 0.01}$ & $1.600_{\pm 0.001}$& $0.200_{\pm 14.30}$ & $0.700_{\pm 0.272}$ &$1.413_{\pm 0.263}$  & $0.696_{\pm 0.267}$ & $2.507_{\pm 1.957}$ 
    \\
    \midrule
    MLE  & \bm{$-0.02_{\pm 0.10}$} & \bm{$0.028_{\pm 0.015}$} & $\bm{0.014_{\pm 0.002}}$ & {$0.012_{\pm 0.006}$} &$\bm{0.098_{\pm 0.107}}$ & $\bm{0.017_{\pm 0.019}}$ & \bm{$0.051{\pm 0.049}$}\\
    \bottomrule
    % \midrule
    % Exp Hawkes & $-0.33_{\pm 0.03}$ & $-0.28_{\pm 0.05}$& $-0.32_{\pm 0.02}$ & $-0.35_{\pm 0.05}$ \\
    % \midrule
    % Sin Hawkes & $-1.11_{\pm 0.04}$ & $-1.04_{\pm 1.02}$ & $-0.74_{\pm 0.02}$ & $-0.90_{\pm 0.03}$\\
\end{tabular}}
\end{sc}
\end{table*}

\subsection{Deep Point Processes Models}
% \subsection{Real Data}

\paragraph{Datasets}
We consider four real datasets. 
(1) \textbf{Half-Sin Hawkes Process}: This is a synthetic 2-variate Hawkes process with trigerring kernel $g_{ij}=\alpha_{ij}sin(\tau), \tau\in (0,\pi)$,   $K=2$.
(2) \textbf{StackOverflow}~\citep{jure2014snap}: This dataset has two years of user awards on StackOverflow. Each user received a sequence of badges 
and there are $K = 22$ kinds of badges. 
% (2) \textbf{Earthquake}~\citep{xue2023easytpp}: This dataset includes earthquakes over the Conterminous U.S from 1996 to 2023, with the number of event types $K=7$ indicating the severity of each earthquake. 
(3) \textbf{Retweet}~\citep{zhao2015seismic}: This dataset includes sequences indicating how each novel tweets are forwarded by other users. 
Retweeter categories serve as event types $K=3$. 
% : small ($<$120 followers), medium ($<$1363 followers), and large (the rest). 
(4) \textbf{Taobao}~\citep{xue2022hypro}: This dataset comprises user activities on Taobao (in total $K=17$ event types). For each dataset, we follow the default training/dev/testing split in the repository. 

\paragraph{Training Protocol}
In recent years, many deep point process models have been proposed. Here, we focus on two of the most popular attention-based Hawkes process models: 
% \textbf{RMTPP}~\citep{du2016recurrent}, \textbf{NHP}~\citep{mei2017neural}, 
\textbf{SAHP}~\citep{zhang2020self} and \textbf{THP}~\citep{simiao2020transformer}. 
We deploy AWSM and ASM on THP and SAHP. 
For each dataset, we train 3 seeds with the same epochs and report the mean and standard deviation of the best TLL and ACC.
% For each dataset, we train 3 times with 30 epochs each time and report the mean and standard deviation of TLL and ACC. 
% For the real data, the true observation endpoint $T$ of the whole dataset is unknown. Therefore, we choose the maximum event time of each batch as the observation endpoint. 
When using MLE, we adopt numerical integration to calculate the intensity integral. To ensure model accuracy, the number of integration nodes is set to be large enough as we sample 10 nodes between every two adjacent events. When using DSM, we tune the variance of noise for better results. When using AWSM, since for real datasets, the true observation endpoint $T$ is unknown. We choose the maximum event time of each batch as the observation endpoint for weight function $\bm h^0$. This may lead to unsatisfying results since real datasets may not be sampled during a unified time window. We provide a remedy for this as discussed in \cref{sec:seq_trunc}. Details of training and testing hyperparameters are provided in \cref{sec:hyperparameters}.
% Such nodes are dense enough for the accurate estimation of the intensity integral. 

\paragraph{Results}
In \cref{table: Real Data Experiment}, we report the performance of SAHP and THP trained using three different methods, namely MLE, AWSM, and DSM, on four datasets. It is evident from the results that models trained with MLE and AWSM exhibit very similar performance in terms of both TLL and ACC on the test data. This indicates consistency between MLE and AWSM, as they yield comparable model parameters. For DSM, it is significantly inferior to the performance of MLE and AWSM. This may result from the fact that the DSM objective is a biased estimation of the original SM objective and fails to produce consistent estimation when $\sigma>0$ as discussed in \citep{vincent2011connection}.
For ASM, it completely fails in the scenarios mentioned above. It is unable to estimate the correct parameters, and its results are not reported. 
Generally, for complex point process models such as deep Hawkes processes, the necessary regularity conditions are not satisfied, meaning that 
% the last two terms in \cref{eq:nuisance term} are non-zero. 
% Thus 
ASM's objective is incomplete.

% In \cref{table: Real Data Experiment}, we report the performance of SAHP and THP trained using two different methods, namely MLE and AWSM, on four real datasets. It is evident from the results that models trained with MLE and AWSM exhibit very similar performance in terms of both TLL and ACC on the test data. This indicates consistency between MLE and AWSM, as they yield comparable model parameters. 
% However, ASM completely fails in the scenarios mentioned above. It is unable to estimate the correct parameters, and its results are not reported. 
% Generally, for complex point process models such as deep Hawkes processes, the necessary regularity conditions are not satisfied, meaning that 
% % the last two terms in \cref{eq:nuisance term} are non-zero. 
% % Thus 
% ASM's objective is incomplete. 

\begin{figure*}[t]
    \begin{center}
    \adjustbox{valign=b}{
    \begin{minipage}{0.23\linewidth}
    \includegraphics[width=\columnwidth]{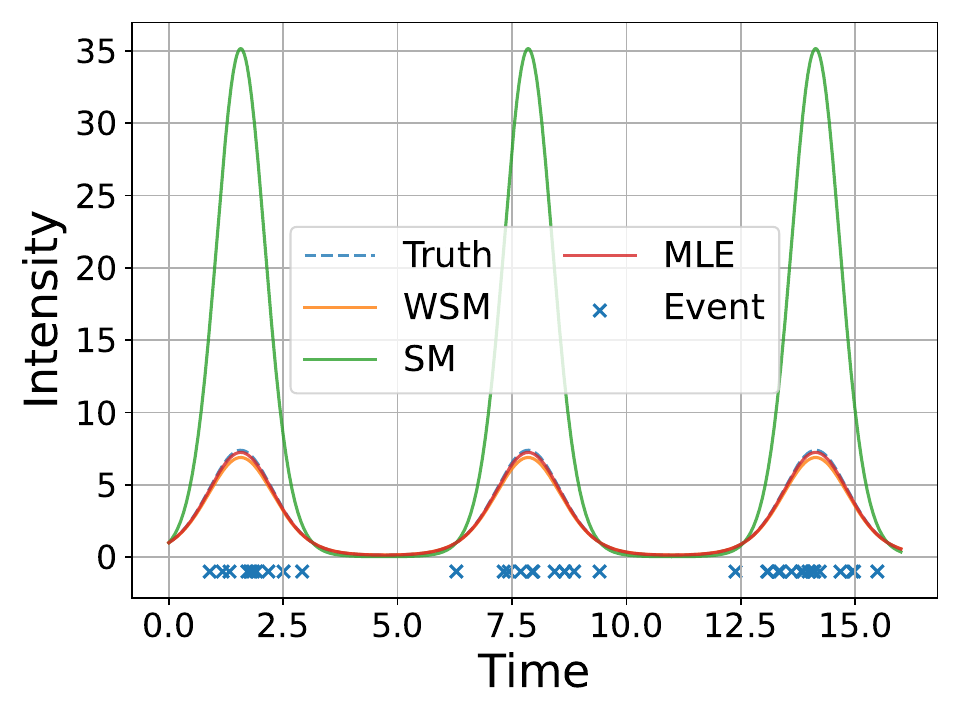}
    \subcaption{Poisson}
    %\label{fig: Exponential Intensity}
    \end{minipage}}
    \adjustbox{valign=b}{
    \begin{minipage}{0.23\linewidth}
    \includegraphics[width=\columnwidth]{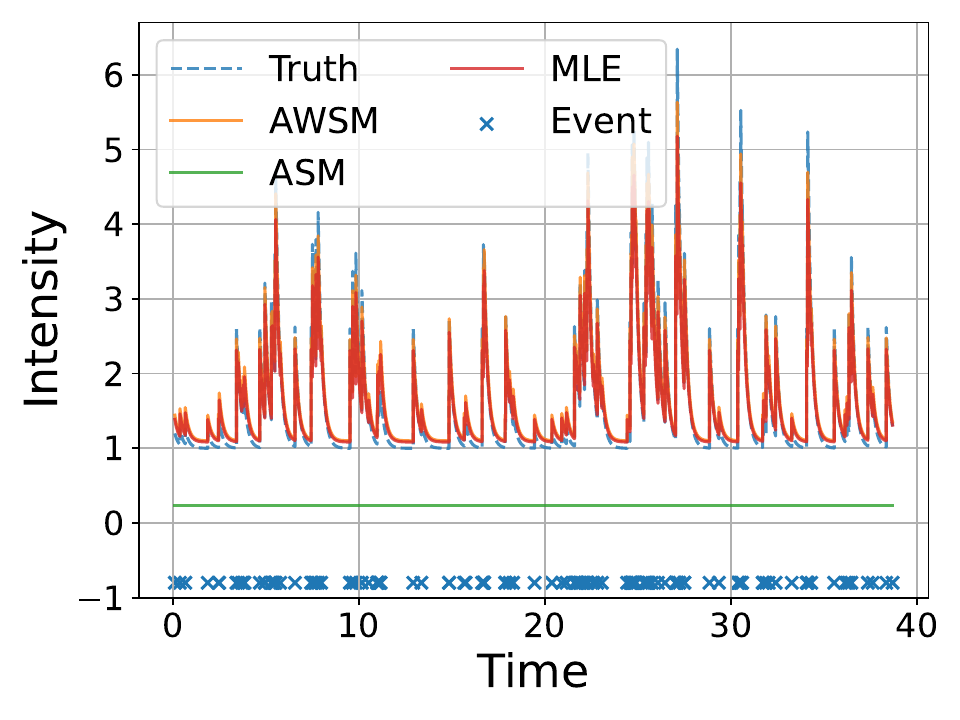}
    \subcaption{Exp-Hawkes Dim-1}
    %\label{fig: Sinusoidal Intensity}
    \end{minipage}}
    \adjustbox{valign=b}{
    \begin{minipage}{0.23\linewidth}
    \includegraphics[width=\columnwidth]{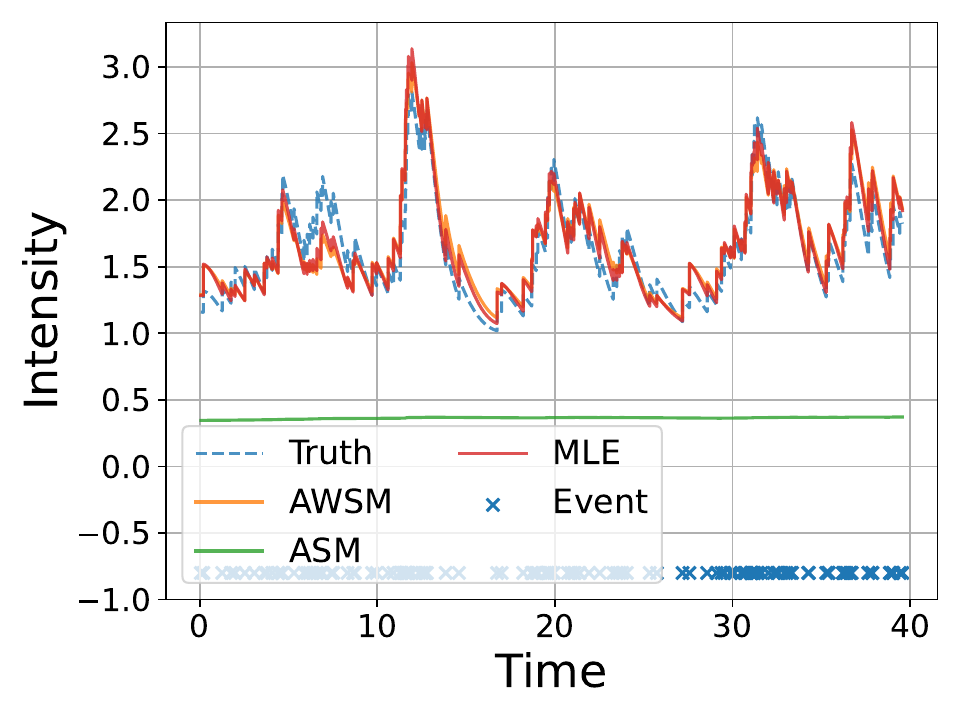}
    \subcaption{Gau-Hawkes Dim-1}
    %\label{fig: Sinusoidal Intensity}
    \end{minipage}}
    \adjustbox{valign=b}{
    \begin{minipage}{0.23\linewidth}
    \label{fig:runtime}
    \includegraphics[width=\columnwidth]{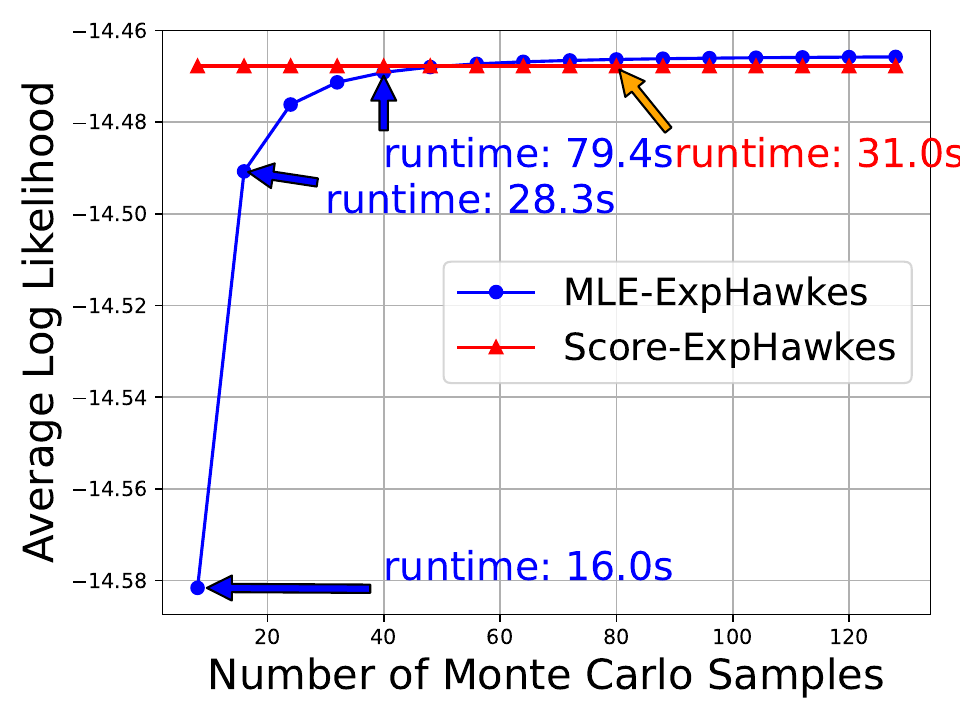}
    \subcaption{TLL and Runtime}
    \label{mle_wsm}
    \end{minipage}}
    \end{center}
    \caption{The learned intensity functions from MLE, (A)SM, and (A)WSM on (a) Poisson, (b) Exp-Hawkes and (c) Gaussian-Hawkes. 
    % The ground-truth intensity is very close to the estimated one, so sometimes it is obscured.
    We present the results for the 1-st dimension. The 2-nd dimension are in \cref{sec:additional experiment}. 
    The ground truth, MLE, and (A)WSM nearly overlap, while (A)SM differs. 
    (d) The TLL and runtime of (A)WSM and MLE w.r.t. the number of integration nodes.}
    \label{learned_intensity}
\end{figure*}

\begin{table*}[t]
\centering
\caption{The TLL and ACC of two attention-based deep Hawkes process models trained by MLE and AWSM on four datasets.
Because ASM estimator completely fails, we do not report its results.}
\label{table: Real Data Experiment}

\begin{sc}
\scalebox{0.70}{
\begin{tabular}{c|ccc|ccc}
    \toprule
    \multirow{2}{*}{Dataset} & \multicolumn{3}{c}{SAHP (TLL$\uparrow$)}  & \multicolumn{3}{c}{THP (TLL$\uparrow$)} \\
    \cmidrule{2-7}
        & MLE & AWSM & DSM & MLE & AWSM & DSM \\
    % \midrule
    % Exp Hawkes & $-0.33_{\pm 0.03}$ & $-0.28_{\pm 0.05}$& $-0.32_{\pm 0.02}$ & $-0.35_{\pm 0.05}$ \\
    % \midrule
    % Sin Hawkes & $-1.11_{\pm 0.04}$ & $-1.04_{\pm 1.02}$ & $-0.74_{\pm 0.02}$ & $-0.90_{\pm 0.03}$\\
    \midrule
    Half-Sin & $1.542_{\pm 0.038}$& \bm{$1.703_{\pm 0.014}$}& $0.804_{\pm 0.353}$ & {$1.161_{\pm 0.031}$} & \bm{$1.271_{\pm 0.036}$}& $-0.385_{\pm 0.033}$\\
    \midrule
    Stackoverflow & \bm{$-2.428_{\pm 0.14}$} &{$-2.541_{\pm 0.461}$} & $-2.629
_{\pm 0.068}$&\bm{$-2.368_{\pm 0.003}$} & $-2.508_{\pm 0.007}$& $-2.782_{\pm 0.034}$\\
    \midrule
    Tabao & \bm{$-1.050_{\pm 0.100}$}& {$-1.373_{\pm 0.091}$} & $-1.911_{\pm 0.049}$  &$ {-1.052_{\pm 0.012}}$ &\bm {$-0.948_{\pm 0.004}$} & $-1.791_{\pm 0.040}$\\
    \midrule
    Retweets & \bm{$0.454_{\pm 0.009}$} & $0.411_{\pm 0.077}$&  $0.110_{\pm 0.186}$&\bm{$0.421_{\pm 0.012}$} & $0.419_{\pm 0.009}$ & $-0.183_{\pm 0.197}$ \\

    \bottomrule
% \end{tabular}}
% \end{sc}

% \begin{sc}
% \scalebox{0.73}{
% \begin{tabular}{c|ccc|ccc}
    \toprule
    \multirow{2}{*}{Dataset} & \multicolumn{3}{c}{SAHP (ACC$\uparrow$)}  & \multicolumn{3}{c}{THP (ACC$\uparrow$)} \\
    \cmidrule{2-7}
        & MLE & AWSM & DSM & MLE & AWSM &DSM \\
    % \midrule
    % Exp Hawkes & $-0.33_{\pm 0.03}$ & $-0.28_{\pm 0.05}$& $-0.32_{\pm 0.02}$ & $-0.35_{\pm 0.05}$ \\
    % \midrule
    % Sin Hawkes & $-1.11_{\pm 0.04}$ & $-1.04_{\pm 1.02}$ & $-0.74_{\pm 0.02}$ & $-0.90_{\pm 0.03}$\\
    \midrule
    Half-Sin & ${0.502}_{\pm 0.001}$& \bm{$0.505_{\pm 0.001}$}& ${0.501}_{\pm 0.001}$ &${0.508_{\pm 0.016}}$& \bm{$0.523_{\pm 0.010}$} &$0.503_{\pm 0.001}$ \\
    \midrule
    Stackoverflow & $0.461_{\pm 0.001}$& $\bm{0.462_{\pm 0.01}}$& $0.421_{\pm 0.042}$ &{$0.461
_{\pm 0.001}$} & \bm{$0.462_{\pm 0.001}$} & $0.445_{\pm 0.016}$\\
    \midrule
    Tabao &\bm{$0.572_{\pm 0.022}$} & $0.455_{\pm 0.011}$& $0.421_{\pm 0.017}$&  \bm{$0.594_{\pm 0.001}$}& ${0.592_{\pm 0.002}}$& $0.435_{\pm 0.010}$   \\
    \midrule
    Retweets & \bm{$0.454_{\pm 0.009}$}& $\bm{0.411_{\pm 0.077}}$ & $0.590_{\pm 0.009}$ &$\bm{0.594_{\pm 0.001}}$& $0.592_{\pm 0.002}$ & $0.556_{\pm 0.011}$\\

    \bottomrule
\end{tabular}}
\end{sc}

\end{table*}

\subsection{Advantage of (A)WSM over MLE}

% The advantage of SM over MLE is that it eliminates the need to compute intensity integrals. For complex point process models, these integrals often require numerical methods, and the number of integration nodes directly affects MLE accuracy. In contrast, SM is not constrained by this limitation.We evaluate the parameter test logliklihood of MLE and (A)WSM on a EXP-Hawkes dataset as the number of integration nodes varies. As shown in (d) of \cref{learned_intensity}, with a limited number of nodes, MLE is faster but exhibits substantial estimation errors. Increasing the nodes reduces the error but significantly increases computation time. In contrast, (A)WSM, being independent of numerical integration, always maintains consistent performance.
The key advantage of (A)WSM over MLE is its avoidance of computing intensity integrals, which can be computationally intensive for complex point process models and impact MLE accuracy. We evaluate the test log-likelihood of MLE and AWSM on the Exp-Hawkes dataset as the number of integration nodes varies. As shown in \cref{mle_wsm}, with a limited number of nodes, MLE is faster but exhibits substantial estimation errors. Increasing the number of nodes reduces the error but significantly increases computation time. In this scenario, AWSM is much faster than MLE with the same accuracy, thus offering better computational efficiency. 
\subsection{Comparison Between Weights}
Though we provide theoretical insight into the choice of an optimal weight function for AWSM, its validity still needs to be testified by experiments. Here, we compare the near-optimal weight $\bm h^0$ with natural weight $\bm h^1$ and squareroot weight $\bm h^2$ satisfying \cref{eq:awsm_require2},
\begin{equation*}
        h_n^1(t_n)=(t_n-t_{n-1})(T-t_{n}),
        h_n^2(t_n)=\sqrt{(t_n-t_{n-1})(T-t_n)}.
\end{equation*}
All three weight functions can be applied in AWSM to recover ground-truth parameters, however with different convergence rates. We carry out experiments on synthetic data for exponential-decay model with the same setting as \cref{synthetic} in our paper. We measure their MAE for different sample sizes in \cref{fig:compare_weight} and find that $\bf h^0$ does achieve the best results among the three weight functions. 

\begin{figure*}[t]
    \begin{center}
    \adjustbox{valign=b}{
    \begin{minipage}{0.31\linewidth}
    \includegraphics[width=\columnwidth]{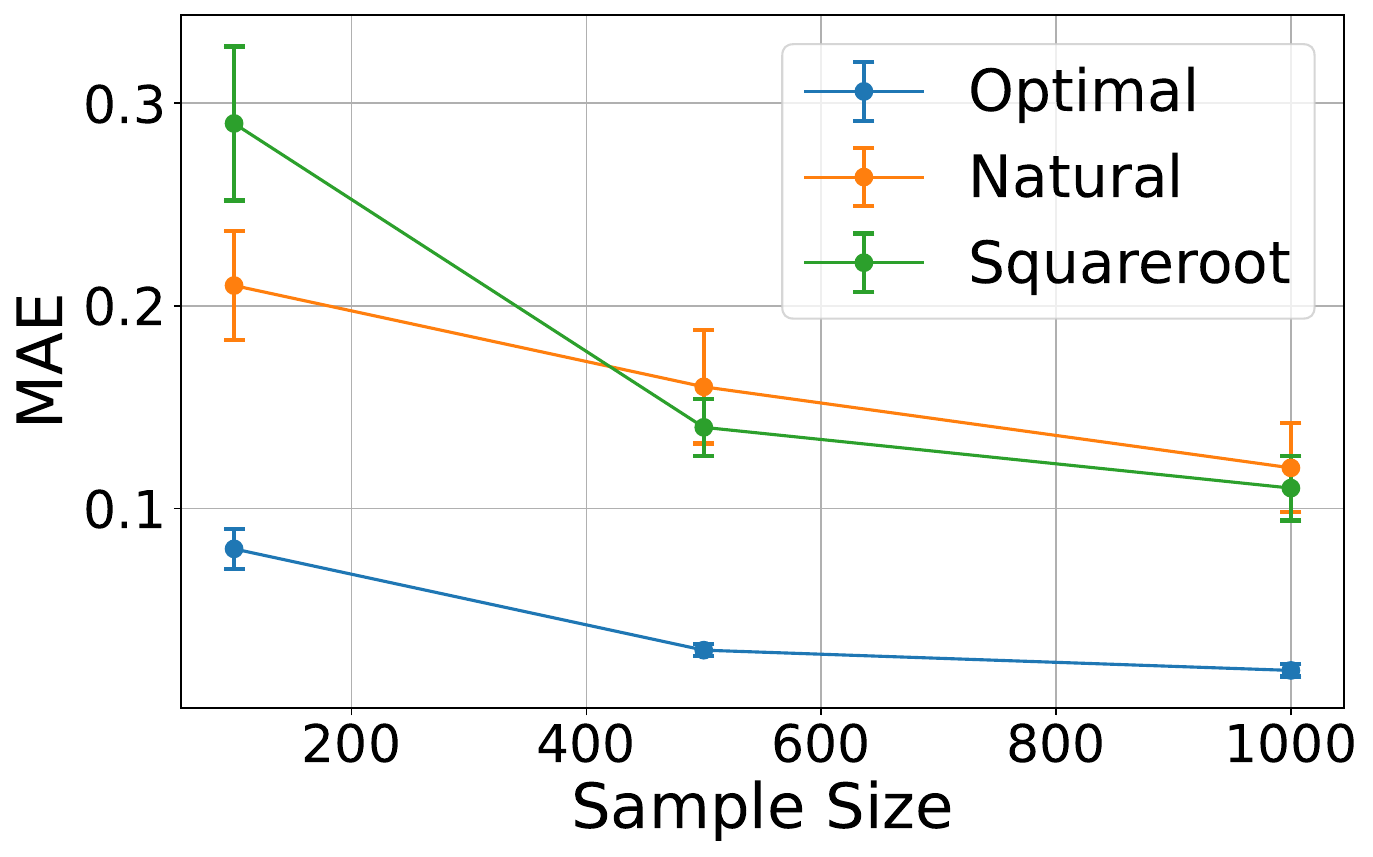}
    \subcaption{On parameter $\alpha_{11}$}
    %\label{fig: Sinusoidal Intensity}
    \end{minipage}}
    \adjustbox{valign=b}{
    \begin{minipage}{0.31\linewidth}
    \includegraphics[width=\columnwidth]{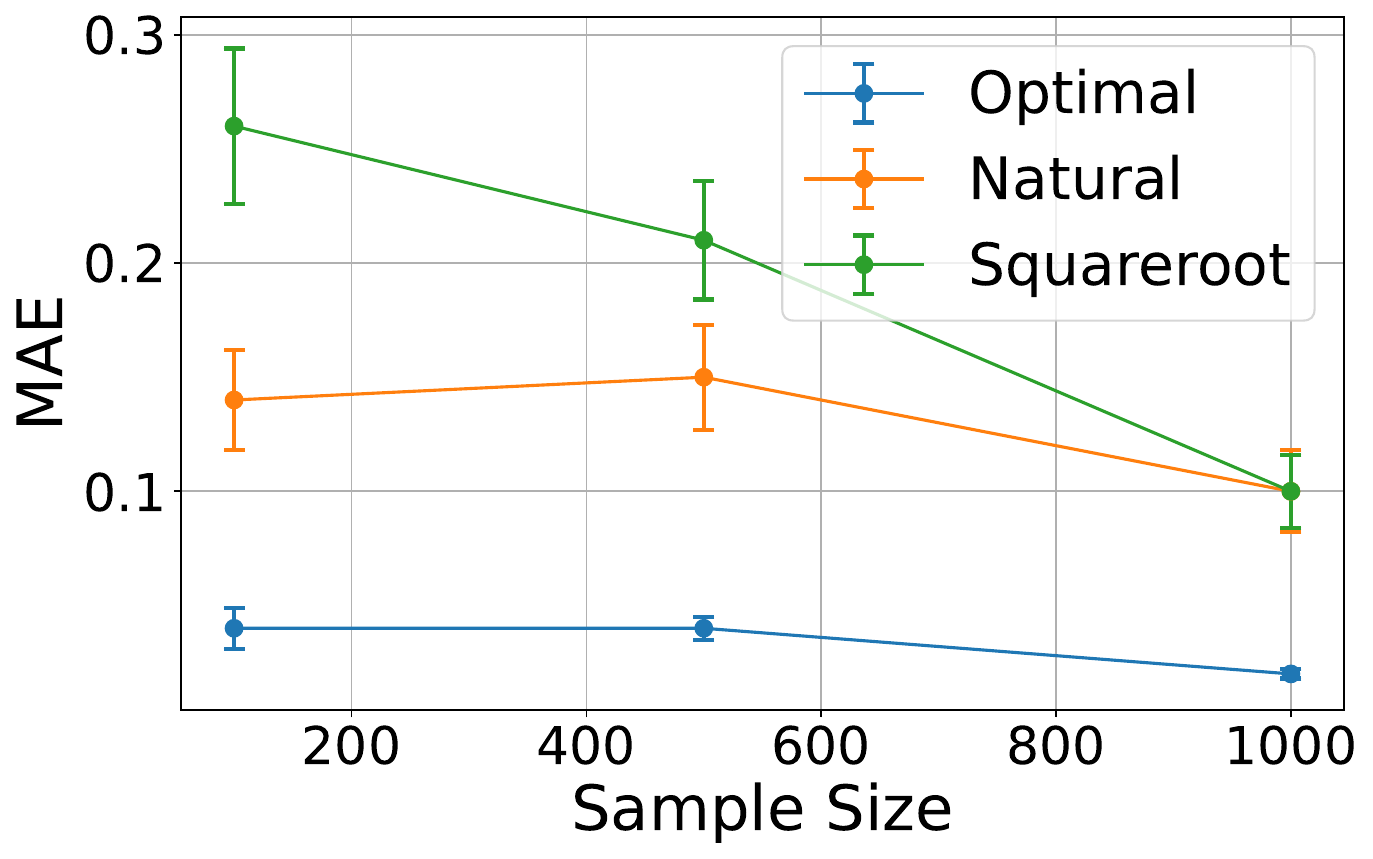}
    \subcaption{On parameter $\alpha_{21}$}
    %\label{fig: Sinusoidal Intensity}
    \end{minipage}}
    \adjustbox{valign=b}{
    \begin{minipage}{0.31\linewidth}
    \includegraphics[width=\columnwidth]{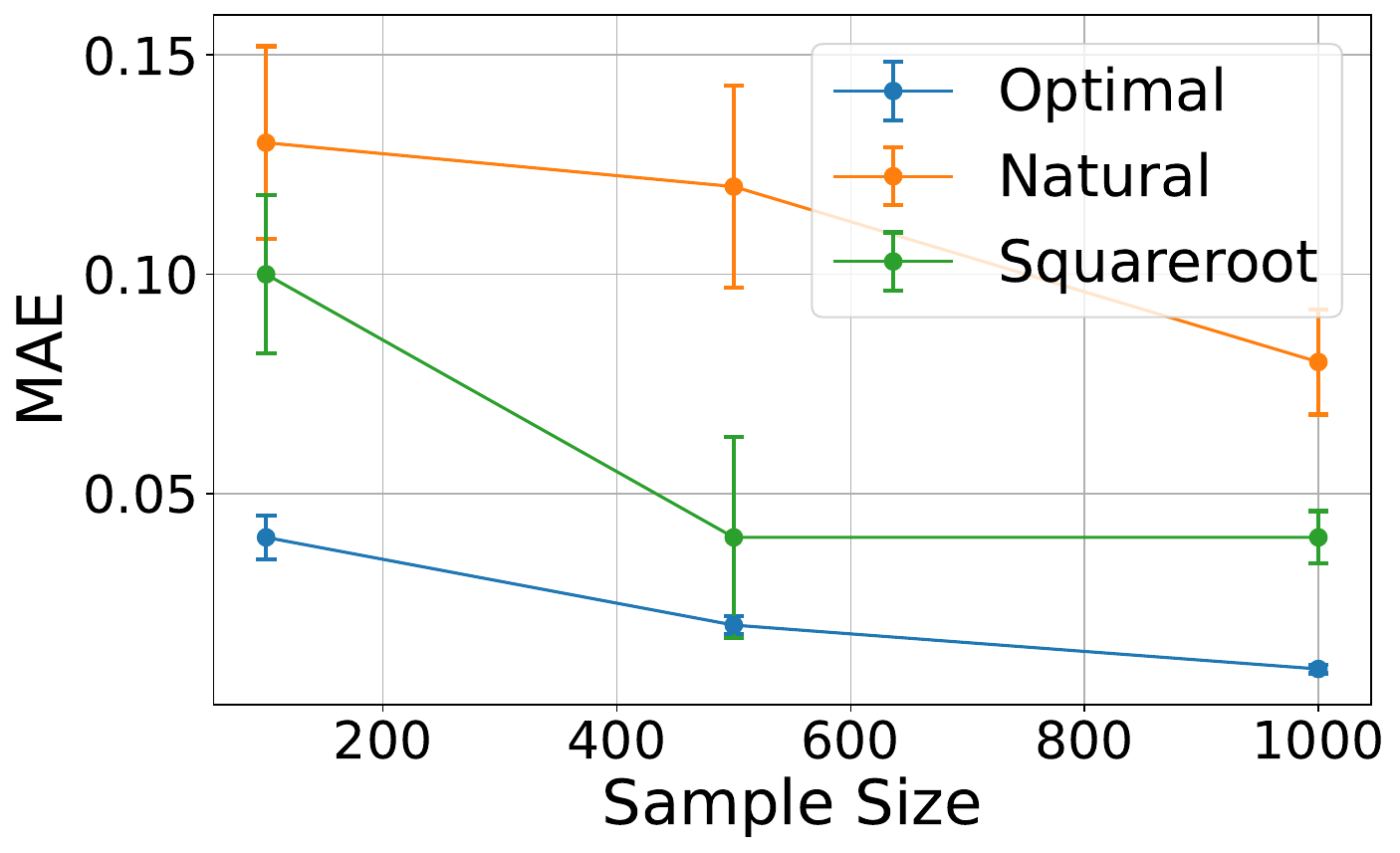}
    \subcaption{On parameter $\mu_1$}
    \end{minipage}}
    \end{center}
    \caption{MAE of parameter estimation versus sample size for three different weight functions on Exponential-Hawkes Model. Our near-optimal weight function outperforms the rest two valid weight functions in all sample sizes. We only show results for three parameters. The rest parameters have almost the same paradigm.}
    \label{fig:compare_weight}
\end{figure*}

\section{Limitations}
\label{limitation}
The current limitation of the methodology is that some real data are collected from multiple time intervals $[0,T_1],\ldots, [0, T_L]$ or collated in a fixed time interval $[0, T]$ with unknown $T$. However, for a score matching to be valid, the required weight function must involve knowledge of $T$. Currently, our remedy including approximate $T$ or performing data truncation as discussed in \cref{sec:seq_trunc}.

% Our first remedy currently is to approximate $T$ based on the maximum time in each batch. This approximation could lead to unsatisfactory performance sometimes. Our second remedy is performing sequence truncation and make each sequence to be of same length, then perform weighted score matching under fixed-length setting. The major drawback of this method is that training data are not fully used because of the truncation. 

% Some works have proposed more efficient SM estimators to eliminate the time-consuming second-order gradient terms. In this work, we focus primarily on how to correct the SM estimator for point processes to ensure it functions properly and on analyzing its statistical properties theoretically. We do not emphasize computational efficiency here, but we intend to improve it in future research. 
\section{Conclusions}
\label{conclusion}
In conclusion, the SM estimator for point processes can overcome the challenges associated with intensity integrals in MLE. While existing works have proposed SM estimators for point processes, our investigation reveals that they prove effective only for specific problems and fall short in more general cases. 
To address this issue, our work introduces a novel approach: the (A)WSM estimator for point processes, offering both theoretical soundness and empirical success.

\begin{ack}
This work was supported by NSFC Project (No. 62106121), the MOE Project of Key Research Institute of Humanities and Social Sciences (22JJD110001), the fundamental research funds for the central universities, and the research funds of Renmin University of China (24XNKJ13). 
\end{ack}

% \section*{Broader Impacts}
% This paper presents work whose goal is to advance the field of machine learning. There are many potential societal consequences of our work, none of which we feel must be specifically highlighted here. 

% \input{neurips_2024.bbl}
\bibliographystyle{plain}
\bibliography{neurips_2024}

\begin{thebibliography}{10}

\bibitem{bacry2014hawkes}
Emmanuel Bacry and Jean-Fran{\c{c}}ois Muzy.
\newblock Hawkes model for price and trades high-frequency dynamics.
\newblock {\em Quantitative Finance}, 14(7):1147--1166, 2014.

\bibitem{chen2013inference}
Feng Chen and Peter Hall.
\newblock Inference for a nonstationary self-exciting point process with an application in ultra-high frequency financial data modeling.
\newblock {\em Journal of Applied Probability}, 50(4):1006--1024, 2013.

\bibitem{hawkes1971spectra}
Alan~G Hawkes.
\newblock Spectra of some self-exciting and mutually exciting point processes.
\newblock {\em Biometrika}, 58(1):83--90, 1971.

\bibitem{hawkes2018hawkes}
Alan~G Hawkes.
\newblock Hawkes processes and their applications to finance: a review.
\newblock {\em Quantitative Finance}, 18(2):193--198, 2018.

\bibitem{Hyvarinen05}
Aapo Hyv{\"{a}}rinen.
\newblock Estimation of non-normalized statistical models by score matching.
\newblock {\em J. Mach. Learn. Res.}, 6:695--709, 2005.

\bibitem{hyvarinen2007some}
Aapo Hyv{\"a}rinen.
\newblock Some extensions of score matching.
\newblock {\em Computational statistics \& data analysis}, 51(5):2499--2512, 2007.

\bibitem{jure2014snap}
Leskovec Jure.
\newblock Snap datasets: Stanford large network dataset collection.
\newblock {\em Retrieved December 2021 from http://snap. stanford. edu/data}, 2014.

\bibitem{kingma2014adam}
Diederik~P Kingma and Jimmy Ba.
\newblock Adam: A method for stochastic optimization.
\newblock {\em arXiv preprint arXiv:1412.6980}, 2014.

\bibitem{kingman1992poisson}
John Frank~Charles Kingman.
\newblock {\em Poisson processes}, volume~3.
\newblock Clarendon Press, 1992.

\bibitem{li2023smurf}
Zichong Li, Yanbo Xu, Simiao Zuo, Haoming Jiang, Chao Zhang, Tuo Zhao, and Hongyuan Zha.
\newblock Smurf-thp: Score matching-based uncertainty quantification for transformer hawkes process.
\newblock 2023.

\bibitem{linderman2016bayesian}
Scott~Warren Linderman.
\newblock {\em Bayesian Methods for Discovering Structure in Neural Spike Trains}.
\newblock PhD thesis, Harvard University, 2016.

\bibitem{liu2022estimating}
Song Liu, Takafumi Kanamori, and Daniel~J Williams.
\newblock Estimating density models with truncation boundaries using score matching.
\newblock {\em The Journal of Machine Learning Research}, 23(1):8448--8485, 2022.

\bibitem{meng2020autoregressive}
Chenlin Meng, Lantao Yu, Yang Song, Jiaming Song, and Stefano Ermon.
\newblock Autoregressive score matching.
\newblock {\em Advances in Neural Information Processing Systems}, 33:6673--6683, 2020.

\bibitem{milgrom2002envelope}
Paul Milgrom and Ilya Segal.
\newblock Envelope theorems for arbitrary choice sets.
\newblock {\em Econometrica}, 70(2):583--601, 2002.

\bibitem{mohler2013modeling}
George Mohler.
\newblock Modeling and estimation of multi-source clustering in crime and security data.
\newblock {\em The Annals of Applied Statistics}, pages 1525--1539, 2013.

\bibitem{ogata1998space}
Yosihiko Ogata.
\newblock Space-time point-process models for earthquake occurrences.
\newblock {\em Annals of the Institute of Statistical Mathematics}, 50(2):379--402, 1998.

\bibitem{ogata1999seismicity}
Yosihiko Ogata.
\newblock Seismicity analysis through point-process modeling: A review.
\newblock In {\em Seismicity patterns, their statistical significance and physical meaning}, pages 471--507. Springer, 1999.

\bibitem{SahaniBM16}
Maneesh Sahani, Gergo Bohner, and Arne Meyer.
\newblock Score-matching estimators for continuous-time point-process regression models.
\newblock In Francesco A.~N. Palmieri, Aurelio Uncini, Kostas~I. Diamantaras, and Jan Larsen, editors, {\em 26th {IEEE} International Workshop on Machine Learning for Signal Processing, {MLSP} 2016, Vietri sul Mare, Salerno, Italy, September 13-16, 2016}, pages 1--5. {IEEE}, 2016.

\bibitem{van2000asymptotic}
Aad~W Van~der Vaart.
\newblock {\em Asymptotic statistics}, volume~3.
\newblock Cambridge university press, 2000.

\bibitem{vincent2011connection}
Pascal Vincent.
\newblock A connection between score matching and denoising autoencoders.
\newblock {\em Neural computation}, 23(7):1661--1674, 2011.

\bibitem{xue2022hypro}
Siqiao Xue, Xiaoming Shi, James Zhang, and Hongyuan Mei.
\newblock Hypro: A hybridly normalized probabilistic model for long-horizon prediction of event sequences.
\newblock {\em Advances in Neural Information Processing Systems}, 35:34641--34650, 2022.

\bibitem{yu2019generalized}
Shiqing Yu, Mathias Drton, and Ali Shojaie.
\newblock Generalized score matching for non-negative data.
\newblock {\em The Journal of Machine Learning Research}, 20(1):2779--2848, 2019.

\bibitem{zhang2020self}
Qiang Zhang, Aldo Lipani, {\"{O}}mer Kirnap, and Emine Yilmaz.
\newblock Self-attentive {H}awkes process.
\newblock In {\em Proceedings of the 37th International Conference on Machine Learning, {ICML} 2020, 13-18 July 2020, Virtual Event}, volume 119 of {\em Proceedings of Machine Learning Research}, pages 11183--11193. {PMLR}, 2020.

\bibitem{zhang2023integration}
Yixuan Zhang, Quyu Kong, and Feng Zhou.
\newblock Integration-free training for spatio-temporal multimodal covariate deep kernel point processes.
\newblock In {\em Thirty-seventh Conference on Neural Information Processing Systems}, 2023.

\bibitem{zhao2015seismic}
Qingyuan Zhao, Murat~A Erdogdu, Hera~Y He, Anand Rajaraman, and Jure Leskovec.
\newblock Seismic: A self-exciting point process model for predicting tweet popularity.
\newblock In {\em Proceedings of the 21th ACM SIGKDD international conference on knowledge discovery and data mining}, pages 1513--1522, 2015.

\bibitem{zhou2022efficient}
Feng Zhou, Quyu Kong, Zhijie Deng, Jichao Kan, Yixuan Zhang, Cheng Feng, and Jun Zhu.
\newblock Efficient inference for dynamic flexible interactions of neural populations.
\newblock {\em Journal of Machine Learning Research}, 23(211):1--49, 2022.

\bibitem{simiao2020transformer}
Simiao Zuo, Haoming Jiang, Zichong Li, Tuo Zhao, and Hongyuan Zha.
\newblock Transformer {H}awkes process.
\newblock In {\em Proceedings of the 37th International Conference on Machine Learning, {ICML} 2020, 13-18 July 2020, Virtual Event}, volume 119 of {\em Proceedings of Machine Learning Research}, pages 11692--11702. {PMLR}, 2020.

\end{thebibliography}

%%%%%%%%%%%%%%%%%%%%%%%%%%%%%%%%%%%%%%%%%%%%%%%%%%%%%%%%%%%%
\clearpage
\appendix

\addcontentsline{toc}{section}{Appendices}
\theoremstyle{plain}
\newtheorem*{assumption1}{Assumption \ref{assump_seperate}}
\newtheorem*{assumption2}{Assumption \ref{assump_continuity}}
\section{Proof of Results in \texorpdfstring{\cref{method}}{Section 3}}

\subsection{Proof of \texorpdfstring{\cref{prop:ESM and ISM for Poisson Process}}{Proof of Proposition 3.1}}\label{seq:proof of prop ESM and ISM for Poisson Process}

\begin{proof}
    
First consider the cross term in $\mathcal L_{\text{SM}}(\theta)$ and expand the expectation,
\begin{equation}\label{eq:Expand Cross Term In SM}
\begin{aligned}
       &\mathbb E_{p(\mathcal T)}\left[\sum_{n=1}^{N_T}\frac{\partial \log p(\mathcal T)}{\partial t_n}\frac{\partial \log p_{\theta}(\mathcal T)}{\partial t_n}\right]\\
       =&\sum_{N=1}^{\infty}\int p(t_1,\ldots, t_N)\sum_{n=1}^N\frac{\partial \log p(t_1,\ldots, t_N)}{\partial t_n}\frac{\partial \log p_{\theta}(t_1,\ldots, t_N)}{\partial t_n}d\mathcal T_{1:N} \\
       =&\sum_{N=1}^{\infty} \int \sum_{n=1}^N \frac{\partial p(t_1,\ldots, t_N)}{\partial t_n}\frac{\partial \log p_{\theta}(t_1,\ldots, t_N)}{\partial t_n}d\mathcal T_{1:N}\\
       =&\sum_{N=1}^\infty \left \{\int \sum_{n=1}^N [p(t_1,\ldots, t_N)\frac{\partial \log p_{\theta}(t_1,\ldots, t_N)}{\partial t_n}]\big|_{t_n=t_{n-1}}^{t_n=t_{n+1}}d\mathcal T_{-n} \right.\\ &\left.-\int\sum_{n=1}^N p(t_1,\ldots, t_N)\frac{\partial^2 \log p_{\theta}(t_1,\ldots, t_N)}{\partial t_n^2}d\mathcal T_{1:N} \right\}. 
\end{aligned}
\end{equation}
The third equation uses an integral-by-part trick. All integrations above are taken within the area of $\{0\leq t_1 \leq \ldots \leq t_N\leq T\}$ or $\{0\leq t_1\leq \ldots \leq t_{n-1}\leq t_{n+1}\leq \ldots \leq t_N\leq T\}$ when $t_n$ has been integrated. 
For the first term in the right side of the last equation, notice that,
\begin{equation}\label{eq:continuity of score function for poisson process}
    \frac{\partial \log p_{\theta}(t_1,\ldots, t_N)}{\partial t_N} = \frac{\partial}{\partial t_n}\log\lambda_{\theta}(t_n),\ \ \ \  \frac{\partial }{\partial t_n}\log\lambda_{\theta}(t_n)\big|_{t_n=t_{n-1}}=\frac{\partial}{\partial t_{n-1}}\log\lambda_{\theta}(t_{n-1})|_{t_{n-1}=t_n}. 
\end{equation}

Therefore, we can see that, for $n\in \{2,\ldots, n\}$, 
\begin{equation*}
\begin{aligned}
    &\int [p(t_1,\ldots, t_N)\frac{\partial \log p_{\theta}(t_1,\ldots, t_N)}{\partial t_{n-1}}]\big|_{t_{n-1}=t_n}d\mathcal T_{-(n-1)} \\
    =& \int [p(t_1,\ldots, t_N)\frac{\partial \log p_{\theta}(t_1,\ldots, t_N)}{\partial t_n}]\big|_{t_n=t_{n-1}}d\mathcal T_{-n}. 
\end{aligned} 
\end{equation*}
Using the above equation, we manage to cancel out most of the terms being summed in the right side of the thrid equation in \cref{eq:Expand Cross Term In SM} and only leave the first and last term, which completes the proof.

\end{proof}

\subsection{Proof of  \texorpdfstring{\cref{thm:validity of WSM objective}}{Theorem 3.2}}
\begin{proof}
    First, since $\mathcal L_{\text{WSM}}(\theta) \geq 0$ and $\mathcal L_{\text{WSM}}(\theta^*) = 0$, we see $\theta^*$ is a minimizer. If there exists another minimizer $\theta_1$, then we have
\begin{equation*}
    \begin{aligned}
        \mathcal L_{\text{WSM}}(\theta) = \sum_{N=1}^\infty \int \sum_{n=1}^N \left[\frac{\partial}{\partial t_n}\log\lambda_{\theta^*}(t_n) - \frac{\partial}{\partial t_n} \log \lambda_{\theta_1}(t_n)\right]^2h_n(\mathcal T)d\mathcal T. 
    \end{aligned}
\end{equation*}

By the definition of $\mathbf{h}(\mathcal T)$, for any $N$, since $\mathbf{h}(\mathcal T) > 0\  a.s.$ elementwisely on $\{0\leq t_1\leq \ldots \leq t_N\leq T\}$. So $\mathcal L_{\text{WSM}}(\theta) = 0$ implies $\frac{\partial \log\lambda_{\theta^*}(t)}{\partial t}$ = $\frac{\partial \log\lambda_{\theta_1}(t)}{\partial t}\ a.e.$ on $[0, T]$. By assumption, this implies $\theta_1 = \theta^*$, which completes the proof.
\end{proof}

\subsection{Proof of \texorpdfstring{\cref{thm:equivalence between EWSM and IWSM}}{Theorem 3.3}}
\begin{proof}
    The proof is basically the same as the proof of \cref{prop:ESM and ISM for Poisson Process}. We first expand the expectation and consider the cross-term in $\mathcal L_{\text{WSM}}(\theta)$,
    \begin{equation}\label{eq:expand the cross term in EWSM}
        \begin{aligned}
            & \mathbb E_{p(\mathcal T)} \left[\sum_{n=1}^{N_T} \frac{\partial \log p(\mathcal T)}{\partial t_n}\frac{\partial \log p_{\theta}(\mathcal T)}{\partial t_n}h_n(\mathcal T)\right]=\sum_{N=1}^\infty \int \sum_{n=1}^N\frac{\partial p(\mathcal T)}{\partial t_n}\frac{\partial \log p_{\theta}(\mathcal T)}{\partial t_n}h_n(\mathcal T)d\mathcal T_{1:N}\\
            &=\sum_{N=1}^{\infty} \left\{\int \sum_{n=1}^N p(t_1,\ldots, t_N)\frac{\partial \log p_{\theta}(\mathcal T)}{\partial t_n}h_n(\mathcal T)\Big|_{t_n=t_{n-1}}^{t_n=t_{n+1}}d\mathcal T_{-n}\right.\\
            &\left.-\int p(t_1,\ldots, t_N)\sum_{n=1}^N\left[\frac{\partial^2 \log p_{\theta}(\mathcal T)}{\partial t_n^2}h_n(\mathcal T)+\frac{\partial \log p_{\theta}(\mathcal T)}{\partial t_n}\frac{\partial h_n(\mathcal T)}{\partial t_n}\right]d\mathcal T_{1:N}\right\}. 
        \end{aligned}
    \end{equation}   
We denote $t_{N+1}=T$ and $t_0=0$ here. Using the first two equations in \cref{eq:wsm_condition}, we have: 
\begin{equation*}
\begin{aligned}
         &\int [p(t_1,\ldots, t_N)\frac{\partial \log p_{\theta}(t_1,\ldots, t_N)}{\partial t_{n}}h_n(\mathcal T)]\big|_{t_{n}=t_{n+1}}d\mathcal T_{-n}=0, \forall n\in [N],\\
&\int [p(t_1,\ldots, t_N)\frac{\partial \log p_{\theta}(t_1,\ldots, t_N)}{\partial t_n}h_n(\mathcal T)]\big|_{t_n=t_{n-1}}d\mathcal T_{-n}=0,\forall n\in [N].
\end{aligned}
\end{equation*}
Therefore, the first intractable summation term in \cref{eq:expand the cross term in EWSM} will disappear, and the second term equals $-\mathbb{E}_{p(\mathcal{T})}\left[\sum_{n=1}^{N_T} \frac{\partial}{\partial t_n}\psi_{\theta}(t_n)h_n(\mathcal{T}) + \psi_{\theta}(t_n)\frac{\partial}{\partial t_n}h_n(\mathcal{T})\right]$. The existence of such an expectation is due to the last two equations in \cref{eq:wsm_condition}. Therefore, we complete the proof.

% Similar arguments as discussed in \cref{eq:continuity of score function for poisson process} still holds here, and we can easily verify, by the definition of $\mathbf{h}(\mathcal T)$, for $n=\{2,\ldots, N-1\}$,
% \begin{equation*}
% \begin{aligned}
%     &\int [p(t_1,\ldots, t_N)\frac{\partial \log p_{\theta}(t_1,\ldots, t_N)}{\partial t_{n-1}}h_n(\mathcal T)]\big|_{t_{n-1}=t_n}d\mathcal T_{-(n-1)} \\
%     =& \int [p(t_1,\ldots, t_N)\frac{\partial \log p_{\theta}(t_1,\ldots, t_N)}{\partial t_n}h_n(\mathcal T)]\big|_{t_n=t_{n-1}}d\mathcal T_{-n}. 
% \end{aligned} 
% \end{equation*}

% And since $h_1(\mathcal T) = (\frac{t_1}{t_2})^2$ and $h_N(\mathcal T) =( \frac{T-t_N}{T-t_{N-1}})^2$. We know that $h_1(\mathcal T)\big|_{t_1=t_0=0}=0$ and $h_N(\mathcal T)\big|_{t_N=t_{N+1}=T}=0$. Finally, we manage to cancel out all the intractable integration term in the right side of the second equation for \cref{eq:expand the cross term in EWSM}. After that, all the remaining terms are either a constant (do not contain parameter $\theta$) or can be written in a tractable expectation, which completes the proof. 

\end{proof}

We can see from the proof that, in \cref{eq:wsm_condition}, the first two equations ensure that the integration by parts trick does not produce an intractable term, and the last two equations are simply regularity conditions that ensure all terms are well-defined.

\section{Proof of Results in \texorpdfstring{\cref{method2}}{Section 4}}
\begin{lemma}
    Let $f(t_n, \mathcal F_{t_{n-1}})$ be a function of $t_n, \mathcal F_{t_{n-1}}$, where $n\in \{1,\ldots, N_T\}$. Then we have 
    \begin{equation*}
        \mathbb E_{p(\mathcal T)} \left[\sum_{n=1}^{N_T}f(t_n, \mathcal F_{t_{n-1}}) \right]= \sum_{n=1}^\infty \int p(t_1,\ldots, t_n) f(t_n, \mathcal F_{t_{n-1}})d\mathcal T_{1:n}, 
    \end{equation*}
    where $p(t_1,\ldots, t_n)$ is the density of observing these timestamps, the integration is taken over $\{0\leq t_1\leq \ldots \leq t_n\leq T\}$. 
\label{Lemma:exchange summation}
\end{lemma}

\begin{proof}
    We first expand the expectation and obtain,
    \begin{equation*}
    \begin{aligned}
        &\mathbb E_{p(\mathcal T)} \left[\sum_{n=1}^{N_T}f(t_n, \mathcal F_{t_{n-1}}) \right] = \sum_{N=1}^\infty \int p(t_1,\ldots, t_N) \sum_{n=1}^{N} f(t_n, \mathcal F_{t_{n-1}})d\mathcal T\\
        =&\sum_{N=1}^\infty \Big\{\int \int \sum_{n=1}^{N-1} p(t_1,\ldots, t_n)p(t_{n+1},\ldots, t_N|\mathcal F_{t_n})f(t_n,\mathcal F_{t_{n-1}})d\mathcal T_{n+1:N}d\mathcal T_{1:n} \\
        &+\int p(t_1,\ldots, t_N)p(N_T=N|\mathcal F_{t_N})f(t_N,\mathcal F_{t_{N-1}})d\mathcal T \Big\}. 
    \end{aligned}
    \end{equation*}
At this point, we need to first integrate out $t_{n+1},\ldots, t_N$, which uses 
\begin{equation*}
\int p(t_{n+1},\ldots, t_N|\mathcal F_{t_{n}})d\mathcal T_{n+1:N} = p(N_T = N|\mathcal F_{t_n}).
\end{equation*}
Plug this in and we obtain,
\begin{equation*}
\begin{aligned}
    \mathbb E_{p(\mathcal T)} \left[\sum_{n=1}^{N_T}f(t_n, \mathcal F_{t_{n-1}}) \right] &=\sum_{N=1}^{\infty} \int \sum_{n=1}^N p(t_1,\ldots,t_n)p(N_T=N|\mathcal F_{t_n})f(t_n, \mathcal F_{t_{n-1}})d\mathcal T_{1:n}\\
    &= \sum_{N=1}^\infty \sum_{n=1}^N \int p(t_1,\ldots, t_n)p(N_T=N|\mathcal F_{t_n})f(t_n, \mathcal F_{t_{n-1}})d\mathcal T_{1:n}\\
    &= \sum_{n=1}^\infty \int p(t_1,\ldots, t_n)f(t_n, \mathcal F_{t_{n-1}})\sum_{N=n}^\infty p(N_T=N|\mathcal F_{t_n})d\mathcal T_{1:n}\\
    &=\sum_{n=1}^\infty \int p(t_1,\ldots, t_n)f(t_n,\mathcal F_{t_{n-1}})d\mathcal T_{1:n}. 
\end{aligned}
\end{equation*}
The thrid equation adopts the exchange of summation. The feasibility is ensured by the assumption that the expectation in the left side of the equation exists and Fubini's theorem. The fourth equation use the fact that $\sum_{N=n}^\infty p(N_T=N|\mathcal F_{t_{n}}) = 1$.
    
\end{proof}

\subsection{Proof of \texorpdfstring{\cref{prop:equivalency}}{Propositioin 4.1}}
\begin{proof}
We use \cref{Lemma:exchange summation} to the cross term of $\mathcal L_{\text{ASM}}(\theta)$ and obtain,
\begin{equation*}
\begin{aligned}
    &\mathbb E_{p(\mathcal T)} \left[\sum_{n=1}^{N_T}\psi(t_n|\mathcal F_{t_{n-1}})\psi_{\theta}(t_n|\mathcal F_{t_{n-1}})\right]\\
    &=\sum_{n=1}^\infty \int p(t_1,\ldots t_n)\psi(t_n|\mathcal F_{t_{n-1}})\psi_{\theta}(t_n|\mathcal F_{t_{n-1}})d\mathcal T_{1:n}\\
    &= \int_{t_0}^{T} p(t_1|\mathcal F_{t_0})\frac{\partial \log p(t_1|\mathcal F_{t_0})}{\partial t_1}\psi_{\theta}(t_1|\mathcal F_{t_0})dt_1\\
    &+\sum_{n=2}^{\infty}\int p(t_1,\ldots, t_{n-1})p(t_n|\mathcal F_{t_{n-1}})\frac{\partial \log p(t_n|\mathcal F_{t_{n-1}})}{\partial t_n}\psi_{\theta}(t_n|\mathcal F_{t_{n-1}})d\mathcal T_{1:n}\\
    &=\int_{t_0}^T \frac{\partial p(t_1|\mathcal F_{t_0})}{\partial t_1}\psi_{\theta}(t_1|\mathcal F_{t_0})dt_1+\sum_{n=2}^\infty \int\int_{t_{n-1}}^T p(t_1,\ldots, t_{n-1})\frac{\partial p(t_n|\mathcal F_{t_{n-1}})}{\partial t_n}\psi_{\theta}(t_n|\mathcal F_{t_{n-1}})dt_nd{\mathcal T}_{1:n-1}\\
    &=p(t_1|\mathcal F_{t_{0}})\psi_{\theta}(t_1|\mathcal F_{t_0})\big|_{t_1=t_0}^{t_1=T} - \int_{t_0}^T p(t_1|\mathcal F_{t_{0}})\frac{\partial \psi_{\theta}(t_1|\mathcal F_{t_0})}{\partial t_1}dt_1 \\&+ \sum_{n=2}^\infty \int p(t_1,\ldots, t_{n-1})p(t_n|\mathcal F_{t_{n-1}})\psi_{\theta}(t_n|\mathcal F_{t_{n-1}})\big|_{t_n=t_{n-1}}^{t_n=T}d\mathcal T_{1:n-1} \\
    &- \sum_{n=2}^\infty \int p(t_1,\ldots, t_n)\frac{\partial \psi_{\theta}(t_n|\mathcal F_{t_{n-1}})}{\partial t_n}d\mathcal T_{1:n}\\
    &=\sum_{n=1}^{\infty} \int p(\mathcal T_{:n-1})p(t_n|\mathcal F_{t_{n-1}})\psi_{\theta}(t_n|\mathcal F_{t_{n-1}})\Big|_{t_n=t_{n-1}}^{t_n=T}d\mathcal{T}_{:{n-1}}-\sum_{n=1}^\infty \int p(t_1,\ldots, t_n)\frac{\partial \psi_{\theta}(t_n|\mathcal F_{t_{n-1}})}{\partial t_n}d\mathcal T_{1:n}\\
    &=\sum_{n=1}^{\infty} \int p(\mathcal T_{:n-1})p(t_n|\mathcal F_{t_{n-1}})\psi_{\theta}(t_n|\mathcal F_{t_{n-1}})\Big|_{t_n=t_{n-1}}^{t_n=T}d\mathcal{T}_{:{n-1}}-\mathbb E_{p(\mathcal T)}\left[\sum_{n=1}^{N_T} \frac{\partial \psi_{\theta}(t_n|\mathcal F_{t_{n-1}})}{\partial t_n} \right]. 
\end{aligned}
\end{equation*}
For the fifth equation, we simply rearrange the terms. We recall that the notation $p(\mathcal T_{:0})$ equals one. For the last equation, we use \cref{Lemma:exchange summation} again. This will be sufficient to complete the proof.
\end{proof}

\subsection{Proof of \texorpdfstring{\cref{theorem:unique minimizer}}{Theorem 4.2}}
\begin{proof}
     First, since $\mathcal L_{\text{AWSM}}(\theta) \geq 0$ and $\mathcal L_{\text{AWSM}}(\theta^*) = 0$, we see $\theta^*$ is a minimizer. If there exists another minimizer $\theta_1$, then we have 
    \begin{equation*}
        \begin{aligned}
            \mathcal L_{\text{AWSM}}(\theta_1) = \frac{1}{2}\sum_{N=1}^\infty \int \sum_{n=1}^{N} [\psi_{\theta^*}(t_n|\mathcal F_{t_{n-1}})-\psi_{\theta_1}(t_n|\mathcal F_{t_{n-1}})]^2h_n(\mathcal T)d\mathcal T = 0. 
        \end{aligned}
    \end{equation*}
    By the definition of $h_n(\mathcal T)$, we have $h_n(\mathcal T) > 0,\ a.s.$ on $0\leq t_1 \leq \ldots, \leq t_N\leq T$ for any $n \leq N$ and $N \in \mathbb N_+$. This implies, for any $n \leq N$ and $N \in \mathbb N_+$, $\psi_{\theta^*}(t_n|\mathcal F_{t_{n-1}}) = \psi_{\theta_1}(t_n|\mathcal F_{t_{n-1}}),\ a.e.$ on $\{0\leq t_{n-1}\leq t_n\leq T\}$.
    Therefore we have 
    \begin{equation*}
        \log p_{\theta^*}(t_n|\mathcal F_{t_{n-1}})=\log p_{\theta_1}(t_n|\mathcal F_{t_{n-1}})+C, a.e.,
    \end{equation*}
    And we conclude that $C = 0$ since $\int_{t_{n-1}}^T p_{\theta}(t_n|\mathcal F_{t_{n-1}})d t_n = 1$. Then by assumption, we have $p_{\theta^*}(t_n|\mathcal F_{t_{n-1}})=p_{\theta_1}(t_n|\mathcal F_{t_{n-1}})\ a.e. \Rightarrow \theta_1=\theta^*$. 
\end{proof}

\subsection{Proof of \texorpdfstring{\cref{thm:tractable EWSM}}{Theorem 4.3}}
The proof is basically the same as the proof of \cref{prop:equivalency}. We first use the \cref{Lemma:exchange summation} to the cross term of $\mathcal L_{\text{AWSM}}(\theta)$, 

\begin{equation*}
    \begin{aligned}
        &\mathbb E_{p(\mathcal T)}\left[\sum_{n=1}^{N_T} \psi(t_n|\mathcal F_{t_{n-1}})\psi_{\theta}(t_n|\mathcal F_{t_{n-1}})h_n(\mathcal T)\right]\\
  &=\sum_{n=1}^\infty \int p(t_1,\ldots t_n)\psi(t_n|\mathcal F_{t_{n-1}})\psi_{\theta}(t_n|\mathcal F_{t_{n-1}})h_n(\mathcal T)d\mathcal T_{1:n}\\
  &=\sum_{n=1}^{\infty} \int p(\mathcal T_{:n-1})p(t_n|\mathcal F_{t_{n-1}})\psi_{\theta}(t_n|\mathcal F_{t_{n-1}})h_n(\mathcal T)\Big|_{t_n=t_{n-1}}^{t_n=T}d\mathcal{T}_{:{n-1}}\\&-\sum_{n=1}^\infty \int p(t_1,\ldots, t_n)\left[\frac{\partial \psi_{\theta}(t_n|\mathcal F_{t_{n-1}})}{\partial t_n}h_n(\mathcal T) + \psi_{\theta}(t_n|\mathcal F_{t_{n-1}})\frac{\partial h_n(\mathcal T)}{\partial t_n}\right]d\mathcal T_{1:n}. 
    \end{aligned}
\end{equation*}

Between the second and the third line above, we omit the steps used in the derivation of Proposition 4.1 to make it concise. For the term in the third line above, it will be eliminated using \cref{eq:awsm_require2}. For the term in the fourth line above, using Lemma B.1, we have: 
% In the third line above, we omit the steps used in the derivation of Proposition 4.1 to make it concise. For the equation in the third line, the first term will be eliminated using Eq. 15; for the second term, using Lemma B.1, we have: 
 % The first term in the right side of third equation above will be eliminated using Eq. 15. For the second term in the third equation above, using Lemma B.1, we have,
\begin{equation*}
    \begin{aligned}
         &-\sum_{n=1}^\infty \int p(t_1,\ldots, t_n)\left[\frac{\partial \psi_{\theta}(t_n|\mathcal F_{t_{n-1}})}{\partial t_n}h_n(\mathcal T) + \psi_{\theta}(t_n|\mathcal F_{t_{n-1}})\frac{\partial h_n(\mathcal T)}{\partial t_n}\right]d\mathcal T_{1:n}=\\
&-\mathbb E_{p(\mathcal T)}\left[\sum_{n=1}^{N_T}\frac{\partial \psi_{\theta}(t_n|\mathcal F_{t_{n-1}})}{\partial t_n}h_n(\mathcal T) + \psi_{\theta}(t_n|\mathcal F_{t_{n-1}})\frac{\partial h_n(\mathcal T)}{\partial t_n}\right].
\end{aligned}
\end{equation*}
The existence of the expectation is ensured by the last two terms in \cref{eq:awsm_require2}.

\section{Proof of Results in \texorpdfstring{\cref{sec:Theory}}{Section 5}}

We present all the regularity condition needed for establishing the consistency of our estimator:

\begin{assumption}
\label{assu:poisson_compact}
    The parameter space $\Theta$ is a compact set in $\mathbb R^d$ and contains an open set which contains $\theta^*$.
\end{assumption}
\begin{assumption}\label{assu:hawkes_differentiable}
    Both the $\mu_{\theta}(t)$ and $g_{\theta}(t)$ are twice continuously differentiable w.r.t. $t$  for all $\theta \in \Theta$ and  those derivatives are continuous w.r.t. $\theta$.
\end{assumption}

We remind readers that $\mu_{\theta}(t)$ and $g_{\theta}(t)$ are the mean intensity function and the triggering kernel for a Hawkes process, first defined in \cref{eq:Hawkes intensity}.

\begin{assumption}\label{assu:hawkes_identifiable}
    If both $\mu_{\theta_1}(t) = \mu_{\theta_2}(t)\ a.s. $  and  $g_{\theta_1}(t) = g_{\theta_2}(t)\ a.s. $ on $[0, T]$, then $\theta_1 = \theta_2$. 
\end{assumption}

\subsection{Proof of \texorpdfstring{\cref{thm:AWSM_consistency}}{Theorem 5.1}}
\begin{proof}
    As shown in Theorem 5.7 in \cite{van2000asymptotic}, if we can prove a uniform in probability convergence for $\hat {\mathcal J}_{\text{AWSM}}(\theta)$ to $\mathcal J_{\text{AWSM}}(\theta)$, then using the fact that $\theta^*$ is a unique minimizer of $\mathcal J_{\text{WSM}}(\theta)$ in a compact set in $\mathbb R^d$, the consistency is proved. Therefore, we only prove the uniform in probability convergence.
    
    For any $\theta \in \Theta$, for each sampled sequence $(t_1^{(m)},\ldots, t_{N_m}^{(m)})$, we perceive the following value as a random variable,
    \begin{equation*}
    \begin{aligned}
            \xi_{m} = \sum_{n=1}^{N_m}&\left[\frac{1}{2}\psi_{\theta}(t_n^{(m)}|\mathcal F_{t_{n-1}^{(m)}})h_n(\mathcal T^{(m)})+\frac{\partial}{\partial t_n}\psi_{\theta}(t_n^{(m)}|\mathcal F_{t_{n-1}^{(m)}})h_n(\mathcal T^{(m)})\right. \\+&\left.\psi_{\theta}(t_n^{(m)}|\mathcal F_{t_{n-1}^{(m)}})\frac{\partial}{\partial t_n}h_n(\mathcal T^{(m)})\right].
    \end{aligned}
    \end{equation*}
    One can verify that it is a measurable map from the sample space to the real line, therefore indeed a random variable. And its expectation is $\mathcal{J}_{\text{AWSM}}(\theta)$, which is finite by assumption. Since different sequences are i.i.d. samples with finite expectation, using the weak law of large numbers, we have: 
    \begin{equation}\label{eq:pointwise in probability convergence}
        \hat{\mathcal{J}}_{\text{AWSM}}(\theta) = \frac{1}{M}\sum_{m=1}^M \xi_m \xrightarrow{p} \mathcal{J}_{\text{AWSM}}(\theta), \forall \theta \in \Theta. 
    \end{equation}
    Now we prove that this convergence is uniform in $\Theta$. Similar to \cite{chen2013inference}, we first prove that 

    \begin{equation}\label{eq:desired convergence used in Chen}
        \forall \varepsilon > 0, \exists \delta > 0, s.t. \forall ||\theta_1-\theta_2||<\delta, \mathbb P(|\hat{\mathcal{J}}_{\text{AWSM}}(\theta_1) - \hat{\mathcal{J}}_{\text{AWSM}}(\theta_2)|>\frac{1}{3}\varepsilon) \rightarrow 0, M \rightarrow \infty.
    \end{equation}

    % for any $\epsilon > 0$, there exists $\delta > 0$, such that for all $\lvert \theta_1 - \theta_2\rvert < \delta$, $|\hat{\mathcal{J}}_{\text{WSM}}(\theta_1) - \hat{\mathcal{J}}_{\text{WSM}}(\theta_2)| < \epsilon$ with probability tending to $1$. 

    First, by \cref{assu:hawkes_differentiable}, we know $\lambda_{\theta}(t_n|\mathcal F_{t_{n-1}}$ and $\frac{\partial}{\partial t_n}\lambda_{\theta}(t_n|\mathcal F_{t_{n-1}}$ are continuous w.r.t $\theta$.
    Therefore $\psi_{\theta}(t_n|\mathcal F_{t_{n-1}}) = \frac{\partial}{\partial t_n} \log \lambda_{\theta}(t_n|\mathcal F_{n-1})-\lambda_{\theta}(t_n|\mathcal F_{t_{n-1}})$ and $\frac{\partial}{\partial t_n} \psi_{\theta}(t_n|\mathcal F_{t_{n-1}}) = \frac{\partial^2}{\partial t_n^2} \log \lambda_{\theta}(t_n|\mathcal F_{t_{n-1}})-\frac{\partial}{\partial t_n}\lambda_{\theta}(t_n|\mathcal F_{t_{n-1}})$ are both continuous w.r.t. $\theta$, therefore the $\mathcal J_{\text{AWSM}}(\theta)$ is a continuous function of $\theta$. Since $\Theta$ is a compact set in $\mathbb R^d$, we know $\mathcal J_{\text{AWSM}}(\theta)$ is uniform continuous.

    Therefore, we can bound $|\mathcal J_{\text{AWSM}}(\theta_1) - \mathcal J_{\text{AWSM}}(\theta_2)|$ using $\|\theta_1-\theta_2\|$. Using the result in \cref{eq:pointwise in probability convergence}, we know that for any $\varepsilon > 0$, we can find a uniform $\delta$ so that $|\mathcal J_{\text{AWSM}}(\theta_1) - \mathcal J_{\text{AWSM}}(\theta_2)| < \frac{1}{6}\varepsilon, \forall ||\theta_1-\theta_2||<\delta$. So that,
    \begin{equation*}
    \begin{aligned}
    &\mathbb P(|\hat{\mathcal{J}}_{\text{AWSM}}(\theta_1) - \hat{\mathcal{J}}_{\text{AWSM}}(\theta_2)|>\frac{1}{3}\varepsilon)\\ &<  \mathbb P(|\hat{\mathcal{J}}_{\text{AWSM}}(\theta_1) - {\mathcal{J}}_{\text{AWSM}}(\theta_1)|+|\hat{\mathcal{J}}_{\text{AWSM}}(\theta_2) - {\mathcal{J}}_{\text{AWSM}}(\theta_2)|>\frac{1}{6}\varepsilon)\\
    &< \mathbb P(|\hat{\mathcal{J}}_{\text{AWSM}}(\theta_1) - {\mathcal{J}}_{\text{AWSM}}(\theta_1)|>\frac{1}{12}\varepsilon) \rightarrow 0, M\rightarrow \infty. 
    \end{aligned}
    \end{equation*}

Now we follow exactly the same steps as  \cite{chen2013inference} for the uniform in probability convergence. Since \cref{eq:desired convergence used in Chen} hold, for such a $\delta$ in that equation, since $\Theta$ is a compact set in $\mathbb R^d$, there exists a finite number of open balls with radius $\delta$ whose union covers $\Theta$. Let $\vartheta_1,\ldots,\vartheta_i,\ldots, \vartheta_L$ denote the centers of these balls. We denote $\vartheta_{i(\theta)}$ the center of a ball which contains $\theta$. Since we have 

\begin{equation*}
    \begin{aligned}
        &\mathbb P(\sup_{\theta}|\hat{\mathcal{J}}_{\text{AWSM}}(\theta) - {\mathcal{J}}_{\text{AWSM}}(\theta)|>\varepsilon) \leq \mathbb P(\sup_{\theta}|\hat{\mathcal{J}}_{\text{AWSM}}(\theta) - \hat{\mathcal{J}}_{\text{AWSM}}(\vartheta_{i(\theta)})|>\frac{\varepsilon}{3})\\
        &+\mathbb P(\sup_{\theta}|\hat{\mathcal{J}}_{\text{AWSM}}(\vartheta_{i(\theta)}) -  {\mathcal{J}}_{\text{AWSM}}(\vartheta_{i(\theta)})|>\frac{\varepsilon}{3}) + \mathbb P(\sup_{\theta}|{\mathcal{J}}_{\text{AWSM}}(\theta) - {\mathcal{J}}_{\text{AWSM}}(\vartheta_{i(\theta)})|>\frac{\varepsilon}{3}). 
    \end{aligned}
\end{equation*}

The third term on the right equals $0$ because of its definition and the uniform continuous of $\mathcal J_{\text{AWSM}}(\theta)$. The first term converges to $0, M\rightarrow \infty$ by \cref{eq:desired convergence used in Chen}. For the second term, we write 
\begin{equation*}
    \mathbb P(\sup_{\theta}|\hat{\mathcal{J}}_{\text{AWSM}}(\theta) -  {\mathcal{J}}_{\text{AWSM}}(\vartheta_{i(\theta)})|>\frac{\varepsilon}{3})< \sum_{i=1}^L \mathbb P(|\hat {\mathcal J}_{\text{AWSM}}(\theta_i)-\mathcal J_{\text{AWSM}}(\theta_i)| > \frac{\varepsilon}{3}) \rightarrow 0. 
\end{equation*}

Finally we obtain $\sup_{\theta \in \Theta}|\hat {\mathcal J}_{\text{AWSM}}(\theta) - \mathcal J_{\text{AWSM}}(\theta)| \xrightarrow{p} 0$, which completes the proof. 
    
\end{proof}

\subsection{Proof of \texorpdfstring{\cref{thm:error_bound}}{Theorem 5.4}}

First we restate the technical assumptions needed for the proof.

\begin{assumption1}
       Assume there exists $\alpha>1$ such that,
    \begin{equation*}
        \mathop{inf}_{\theta:||\theta-\theta^*||\geq \delta} \mathcal J_{\mathbf h}(\theta)-\mathcal J_{\mathbf{h}}(\theta^*)\geq C_{\mathbf{h}}\delta^{\alpha}
    \end{equation*}
    holds for any small $\delta$. Here, $C_{\mathbf{h}}$ is a positive constant that depends on the weight function $\mathbf h$ such that $C_{a\mathbf{h}}=aC_{\mathbf{h}}$ for any positive constant $a$.
\end{assumption1}

For this assumption, we assume that the optimal parameter $\theta^*$ is well-separated from other neighbouring parameters in terms of the population objective values. This is a standard assumption for Theorem 5.52 in \citep{van2000asymptotic}. This assumption will satisfy if $\nabla_{\theta}^2 \mathcal J_{\bm h}(\theta)$ is positive definite at $\theta^*$.
% If the map $\theta \mapsto \mathcal J_{\mathbf h}(\theta)$ is twice-differentiable at the point of maximum $\theta^*$, then its first derivative at $\theta^*$ vanishes 
% and a Taylor expansion of $\mathcal J_{\mathbf h}(\theta)$ takes the form,
% \begin{equation*}
% \begin{aligned}
%         \mathcal J_{\mathbf h}(\theta) - \mathcal J_{\mathbf h}(\theta^*) &= \mathcal L_{\mathbf h}(\theta) - \mathcal L_{\mathbf h}(\theta^*)\\
%         &=0 + \frac{1}{2}(\theta-\theta^*)^T V(\theta-\theta^*)+o(\|\theta-\theta^*\|^2).
% \end{aligned}
% \end{equation*}
% Then the assumption holds with a = 2 provided that the second-derivative 
% matrix $V$ is positive definite. 

\begin{assumption2}
    For $\forall n\in \mathbb N^+$, there exists $\dot{A_n}(\mathcal T), \dot{B_n}(\mathcal T)$ such that,
    \begin{equation*}
        \begin{aligned}
            |A_n(\mathcal T, \theta_1) - A_n(\mathcal T, \theta_2)| \leq \dot{A_n}(\mathcal T)||\theta_1-\theta_2||,\ \ \ \ |B_n(\mathcal T, \theta_1) - B_n(\mathcal T, \theta_2)| \leq \dot{B_n}(\mathcal T)||\theta_1-\theta_2||. 
        \end{aligned}
    \end{equation*}
\end{assumption2}
For this assumption, we assume the objective function has certain level of continuity. This assumption will realize if $\mathop{sup}_{\theta \in \Theta}||\nabla_{\theta}A_n(\mathcal T, \theta)||$ exisits for any $n$ and the same condition also applies to $B_n(\mathcal T, \theta)$.  

We also define, 
\begin{equation*}
    J_{\mathbf h}(\mathcal T; \theta) = \sum_{n=1}^{N_T} \left[[{A_n(\mathcal T,\theta)}h_n(\mathcal{T})+ {B_n(\mathcal T,\theta)} \frac{\partial h_n(\mathcal T)}{\partial t_n} \right],
\end{equation*}
we see $\mathbb E[J_{\mathbf h}(\mathcal T; \theta)]=\mathcal J_{h}(\theta)$.

Now we begin the proof of \cref{thm:error_bound}.

\begin{proof}
    First define 
\begin{equation*}
    \dot J_{\mathbf h}(\mathcal T) = \sum_{n=1}^{N_T}\left[\dot{A_n}(\mathcal T)h_n(\mathcal T) + \dot{B_n}(\mathcal T)\frac{\partial h_n(\mathcal T)}{\partial t_n}\right].
\end{equation*}

% Since $||\dot J_{\mathbf h}||_{L^2(P)}=\Gamma(\mathbf h, A, B) < \infty$, 

% The next step is to find an upper bound for $||\dot J_{\mathbf h}||_{L^2(P)}$ which is defined as,
% \begin{equation*}
%     ||\dot J_{\mathbf h}||_{L^2(P)} = \mathbb E_{p(\mathcal T)}[\dot J_{\mathbf h}(\mathcal T)].
% \end{equation*}

% Suppose it is indeed upper bounded by a finite number $\Gamma$. The next step is to 
Next we evaluate the bracketing number of such function class,
\begin{equation*}
    \mathcal F_{\delta} :=\{J_{\mathbf h}(\mathcal T; \theta) - J_{\mathbf h}(\mathcal T;\theta^*)|\theta \in \Theta, ||\theta - \theta^*||\leq \delta\}.
\end{equation*}
a

Denote $\mathcal N_{[]}(\varepsilon, \mathcal F, L^2(P))$ as the bracketing number of $\mathcal F$ with the radius $\varepsilon$ under the norm of $L^2(P)$. $P$ is the underlying probability measure induced by $\mathcal T$. Use Example 19.6 in \cite{van2000asymptotic} we have,
\begin{equation*}
    \mathcal N_{[]}(\varepsilon ||J||_{L^2(P)}, \mathcal F_{\delta}, L^2(P))\leq (1+\frac{4\delta}{\varepsilon})^r
\end{equation*}

Since 
\begin{equation}
    |J_{\mathbf h}(\mathcal T; \theta) - J_{\mathbf h}(\mathcal T;\theta^*)|\leq \dot J_{\mathbf h}(\mathcal T)||\theta-\theta^*||\leq \dot J_{\mathbf h}(\mathcal T)\delta,
\end{equation}
therefore $\dot J_{\mathbf h}(\mathcal T)\delta$ is the envelope of $\mathcal F_{\delta}$, we denote it as $F_{\delta}$. Therefore we obtain,

\begin{equation*}
    \mathcal N_{[]}(\varepsilon ||F_{\delta}||_{L^2(P)}, \mathcal F_{\delta}, L^2(P))= \mathcal N_{[]}(\varepsilon \delta ||\dot J_{\bm h}||_{L^2(P)}, \mathcal F_{\delta}, L^2(P))\leq (1+\frac{4}{\varepsilon})^r.
\end{equation*}

After obtaining this quantity, the next step is to upper bound the empirical process defined as,

\begin{equation*}
    \mathbb G_M \big(J_{\mathbf h}(;\theta)-J_{\mathbf h}(;\theta^*)\big):= \frac{1}{\sqrt M}\sum_{m=1}^M\left[J_{\mathbf h}(\mathcal T_m;\theta)-J_{\mathbf h}(\mathcal T_m;\theta^*) - \mathbb E_{p(\mathcal T)}[J_{\mathbf h}(\mathcal T;\theta)-J_{\mathbf h}(\mathcal T;\theta^*)] \right]
\end{equation*}

Using Corollary 19.35 in \cite{van2000asymptotic}, since we have $||\dot J_{\mathbf h}||_{L^2(P)}=\Gamma(\mathbf h, A, B) < \infty$, then,
\begin{equation}\label{eq_sup_norm_of_Fdelta}
\begin{aligned}
        \mathbb E\left[\mathop{sup}_{f\in \mathcal F_{\delta}}\mathbb G_n(f)\right] &\leq CJ_{[]}(||F||_{L^2(P)}, \mathcal F_{\delta}, L^2(P))\\
        &= C||F||_{L^2(P)}\int_{0}^{1}\sqrt{\log N_{[]}(\varepsilon ||F||_{L^2(P)}, \mathcal F_{\delta}, L^2(P))}d\varepsilon\\
        &\leq C\delta ||\dot m||_{L(P)}\sqrt r \int_{0}^1 \sqrt{\log (1+\frac{4}{\varepsilon})}d\varepsilon \\
        &\leq C^{\prime}\sqrt r \delta \Gamma(\bm h, A, B), 
\end{aligned}
\end{equation}
where $C$ and $C^{\prime}$ are universal constant and $J_{[]}(\cdot, \cdot, \cdot)$ is the entropy integral.

Finally, given that $\hat \theta$ converges to $\theta^*$ in probability, combined with \cref{assump_seperate} and \cref{eq_sup_norm_of_Fdelta}, using Theorem 5.52 in \cite{van2000asymptotic} we have, for $\delta < CK_{\alpha}\frac{\sqrt {r}}{2^{\alpha-1}}\frac{\Gamma(\mathbf{h}, A, B)}{C_\mathbf{h}}$, we have,
\begin{equation*}
    \text{Pr}\left[||\hat \theta - \theta^*|| \leq \left(CK_{\alpha}\frac{\Gamma(\mathbf{h}, A, B)}{\delta C_\mathbf{h}}\sqrt{\frac{r}{n}}^{1/(\alpha-1)}\right)\right]\geq 1-\delta,
\end{equation*}
where $C$ is a universal constant and $K_{\alpha}=\frac{2^{2\alpha}}{2^{\alpha-1}-1}$. 

\end{proof}

\subsection{Proof of \texorpdfstring{\cref{thm:optimal_denominator}}{Theorem 5.5}}
First we give a rigorous definition of $\mathcal H$,
\begin{equation}\label{eq:definition_of_H}
\begin{aligned}
\mathcal{H} := &\left\{ \mathbf{h}(\mathcal{T}) \mid \mathbf{h}: \mathbb{R}_+^{N_T} \rightarrow \mathbb{R}_+^{N_T}, \ h_n(\mathcal{T}) = h_n(t_n, \ldots, t_1), h_n(\mathcal T)|_{t_n=t_{n-1}}=h_n(\mathcal T)|_{t_n=T} = 0 \right. \\
&\left. |h_n(t^1_n, t_1, \ldots, t_{n-1}) - h_n(t^2_n, t_1, \ldots, t_{n-1})| \leq |t^1_n - t^2_n|,\forall n, t_n^1, t_n^2 \right\}.
\end{aligned}
\end{equation}

In other words, $\mathcal H$ contains functions whose component function $h_n$ is $1$-Lipschitz continuous w.r.t. its last dimension $t_n$. It's easy to verfity $\mathbf h^0\in \mathcal H$. 

\begin{proof}
First  we prove that,
\begin{equation}\label{eq_max_esm}
    \mathop{max}_{\mathbf h\in \mathcal H}\mathcal L_{\mathbf h}(\theta)= \mathcal L_{\mathbf h^0}(\theta)
\end{equation}

Notice that,
\begin{equation*}
    \mathcal L_{\mathbf h}(\theta)= \frac{1}{2}\mathbb E_{p(\mathcal{T})}\left[\sum_{n=1}^{N_T}(\psi(t_n|\mathcal F_{t_{n-1}}) -\psi_\theta(t_n|\mathcal F_{t_{n-1}}))^2h_n(\mathcal T)\right]
\end{equation*}

Notice that, in $\mathcal H$,  for any $z=t-{n-1}$ or $z=T$,
\begin{equation*}
    |h_n(\mathcal T)| = |h_n(t_n, \ldots, t_1)|=|h_n(t_n,\ldots, t_1)-h_n(z,\ldots, t_1)|\leq |t_n^1-z|=h_n^0(\mathcal T)
\end{equation*}

Therefore,
\begin{equation*}
    \mathcal L_{\mathbf h}(\theta) \leq  \frac{1}{2}\mathbb E_{p(\mathcal{T})}\left[\sum_{n=1}^{N_T}(\psi(t_n|\mathcal F_{t_{n-1}}) -\psi_\theta(t_n|\mathcal F_{t_{n-1}}))^2h^0_n(\mathcal T)\right]=\mathcal L_{\mathbf h^0}(\theta), \forall \mathbf h\in \mathcal H
\end{equation*}
and equation \ref{eq_max_esm} is proved.

Finally, we have,

\begin{equation*}
    \begin{aligned}
        &\text{inf}_{\theta:||\theta-\theta^*||\geq \delta} \mathcal J_{\mathbf h^0}(\theta)  - \mathcal J_{\mathbf h^0}(\theta^*)\\
        =&\text{inf}_{\theta:||\theta-\theta^*||\geq \delta} \mathcal L_{\mathbf h^0}(\theta)  - \mathcal L_{\mathbf h^0}(\theta^*)\\
        =&\text{inf}_{\theta:||\theta-\theta^*||\geq \delta} \text{sup}_{\mathbf h\in \mathcal H} [\mathcal L_{\mathbf h}(\theta)  - \mathcal L_{\mathbf h}(\theta^*)]\\
        =&  \text{inf}_{\theta:||\theta-\theta^*||\geq \delta} \text{sup}_{\mathbf h\in \mathcal H} [\mathcal J_{\mathbf h}(\theta)  - \mathcal J_{\mathbf h}(\theta^*)]\\
        \geq &\text{sup}_{h\in \mathcal H} \text{inf}_{\theta:||\theta-\theta^*||\geq \delta} [\mathcal J_{\mathbf h}(\theta)  - \mathcal J_{\mathbf h}(\theta^*)].
    \end{aligned}
\end{equation*}

The second equality is due to \cref{eq_max_esm} and $\mathcal L_{\mathbf h}(\theta^*)=0$. The inequality is due to max-min inequality.
\end{proof}

\subsection{Continued Discussion on Optimal Weight Function}\label{sec:heuristic_disc_on_optimality}
Given $\bm h^0$ maximizes $C_{\bm h}$, it is still unclear whether it is still preferable considering $\Gamma(\mathbf{h},A,B)$. This is indeed a tough question that we do not yet have a satisfying answer.
For specific parametric models, $\dot{A}_n(\mathcal T), \dot{B}_n(\mathcal T)$ can be computed analytically (see line 478-480) and then $\Gamma(\mathbf{h},A,B)$ can be computed via Monte Carlo. Then we can study how sensitive $\Gamma(\mathbf{h},A,B)$ is to $\bf h$. For general models, especially when $\psi_{\theta}$ is a deep neural network like THP or SAHP, $\Gamma(\mathbf{h},A,B)$ is intractable to compute. 

However, heuristically speaking, our choice of near-optimal weight function $\bf h^0$ should be a good choice even concerning $\Gamma(\mathbf{h},A,B)$. To make $\Gamma(\mathbf{h},A,B)$ small, a natural idea is to make $|h_n(\mathcal T)|$ and $|\frac{\partial}{\partial t_n}h_n(\mathcal T)|$ small. The weight function we chose and its derivative have relatively positive low powers with respect to $t_n$, therefore making $|h_n^0(\mathcal T)|$ and $|\frac{\partial}{\partial t_n}h_n^0(\mathcal T)|$ small. For  weight functions  $h_n^1(\mathcal T)=(T-t_n)(t_n-t_{n-1})$, the power w.r.t. $t_n$ is two. And for $h^2_n(\mathcal T)=\sqrt{(T-t_n)(t_n-t_{n-1})}$, its derivative is $\frac{\partial}{\partial t_n}h_n^2(\mathcal T)=\frac{T-t_n-(t_n-t_{n-1})}{\sqrt{(T-t_n)(t_n-t_{n-1})}}$ , the numerator is usually a bounded quantity and the denominator may be close to zero, making its derivative large. In conclusion, $\bf h^0$ is a better choice compared with $\bf h^1$ or $\bf h^2$ concerning $\Gamma(\mathbf{h},A,B)$.

% \subsection{Further discussion on choosing weight function}

\subsection{Discussion on Poisson Process}
\label{section:theory_poisson}

For Poisson process, results in \cref{sec:Theory} also holds, including the consistency and the convergence rate. For the choice of weight function, a reasonable choice is still the distance function presented below.

Consider a Poisson process, since the support of $\mathcal T$ is $V=\{\mathcal T\in N^{T}, 0\leq t_1\leq \ldots \leq t_{N_T}\leq T\}$.  Define $A \in \mathbb R^{(N_T+1)\times N_T}$ to be a coefficient matrix used below as,
\begin{equation*}
    A = \begin{bmatrix}
-1 & 0 & 0 & \ldots & 0 \\
1 & -1 & 0 & \ldots & 0 \\
0 & 1 & -1 & \ldots & 0 \\
\vdots & \vdots  & \vdots & \ddots & 0 \\
0 & 0 & 0 & 1 & -1\\
0 & 0 & 0 & 0 & 1
\end{bmatrix}.
\end{equation*}

Denote $\mathbf a_n \in \mathbb R^{N_t}$ as the $n$-th row vector of $A$. Let $b_n=0, n\in [N_T], b_{N_T+1}=-T$. Then $V$ is a convex Polytope which can be rewritten as,
\begin{equation*}
    \begin{aligned}
        V = \{\mathcal T \in \mathbb R^{N_T}|\langle\mathbf a_n, \mathcal T\rangle + b_n \leq 0, n=1,2,\ldots, N_{T}+1\}.
    \end{aligned}
\end{equation*}

Therefore, the distance between $\mathcal T$ and $\partial V$ is ,
\begin{equation}\label{eq:distance_of_poisson}
    dist(\mathcal T, \partial V)=\min_{\mathbf z}\{\|\mathcal T - \mathbf z\| \vert \max_{n\in [N_{T}+1]}[\langle \mathbf a_n, \mathbf z\rangle + b_n] = 0\}=\min_{n\in [N_{T}+1]}\frac{|\langle \mathbf a_n, \mathbf z\rangle + b_n|}{\|\mathbf a_n\|}.
\end{equation}

Let $h^0_n(\mathcal T)= dist(\mathcal T, \partial V), \mathbf h^0(\mathcal T)=(h^0_n(\mathcal T), h^0_n(\mathcal T),\ldots, h^0_n(\mathcal T))^T$. Again, $\mathbf h^0(\mathcal T)$ is only weak derivative. The envelope theorem in \citep{milgrom2002envelope} yields,
\begin{equation*}
    \nabla_{\mathcal T}\mathbf h^0(\mathcal T)=-\frac{\mathbf a_{n^*}}{\|\mathbf a_{n^*}\|}, 
\end{equation*}
wehre $n^*$ is the is the minimizer of the last minimization in \cref{eq:distance_of_poisson}. 

Similar arguments as in \cref{section:optimal_weight_hawkes} also applies to Poisson process, with the distance weight function defined above. Such a weight function also maximizes the denominator for the convergence rate of the Poisson process. However, in the experiment, we adtops another weight function for easier implementation defined as $\mathbf h^1(\mathcal T)$ with its $n$-th component
$h^1_n(\mathcal T) = (T-t_n)(t_n-t_{n-1})$.

% \subsection{Discussion on Domain Transformation}

% The distance between $\mathcal T$ and $\partial V$ is,
% \begin{equation*}
    
% \end{equation*}

\section{Additional Experimental Details}
\label{sec:additional experiment}

In this section, we present some experimental details. First we discuss our modification to the original AWSM for when $T$ for a dataset is not accessible. Then we provide addtional estimation results for the parametric model and hyperparameters for results in \cref{table: Real Data Experiment}.

\subsection{Weight for Equal Length Sequences}\label{sec:seq_trunc}
In this paper, we consider the setting of the temporal point process in a fixed time interval $[0, T]$ consistently. Technically speaking, this poses extra difficulty since the dimension(number of events) would also be random. Therefore, the realized point processes should have a random length in this setting, which is the case for our dataset. However, in real data collection, different sequences may be sampled from different lengths of time interval lengths, and our proposed weight may fail under such a case, resulting in inconsistent estimation. In such a case, we propose to truncate the sequences to be of same length, and introduce the weight function for fixed-length temporal point processes.

We focus on AWSM for Hawkes process. A fixed $N$-length Hawkes process could be considered as a $N$ dimensional random variable with density function as,

\begin{equation*}
    p(t_1,\ldots, t_N)=\prod_{n=1}^N \lambda^*(t_n)\int_{0}^{t_N} \lambda^*(t)dt.
\end{equation*}

Since the number of observations is fixed, we could simply treat it as an $N$ dimensional random variable and perform autoregressive weighted score matching. The valid weight function in this case should satisfy $h_n(\mathcal T)|_{t_n = t_{n-1}} = h_n(\mathcal T)|_{t_n = t_{n+1}} = 0$. Similar nonasymptotic bounds could be derived here and a near optimal weight function would be the Euclidean distance between $t_n$ and the boundary of $t_n$'s support, which would be $\{t_{n-1}, t_{n+1}\}$. Therefore, the near-optimal weight function would be $h^0_n(\mathcal T)=\min \{t_n-t_{n-1}, t_{n+1} - t_n\}$.

During experiments, when $T$ is unknown and cannot be approximated well, we perform a data truncation for each batch. We naively drop all timestamps beyond the length of the shortest sequence to make length consistent within a batch. Then we perform an AWSM for the fixed-length Hawkes process. We specify the results with this modification when we show our hyperparameters.

\subsection{Addtional Results and Hyperparameters}
\label{sec:hyperparameters}
\begin{table*}[t]
\centering
\caption{The MAE of 2-variate Hawkes processes trained by MLE, (A)SM, and (A)WSM on the synthetic dataset.}
\label{table:additional_experiments}
\begin{sc}
\scalebox{0.74}{
\begin{tabular}{c|ccc|cccc}
    \toprule
    \multirow{2}{*}{Estimator} & \multicolumn{3}{c|}{Exp-Hawkes }  & \multicolumn{4}{c}{Gaussian-Hawkes } \\
    \cmidrule{2-8}
        & $\alpha_{21}$ & $\alpha_{22}$ & $\mu_{2}$ &  $\alpha_{12}$ & $\alpha_{21}$ & $\alpha_{22}$ & $\mu_2$ \\
    
    \midrule
    (A)WSM & $\bm{0.052_{\pm 0.054}}$ & $\bm{0.022_{\pm 0.005}}$& $\bm{0.011_{\pm 0.012}}$ &${0.037_{\pm 0.043}}$ &$0.082_{\pm 0.080}$  & {$0.012_{\pm 0.010}$} &$0.060_{\pm 0.066}$ 
    \\
     \midrule
    (A)SM & $0.769_{\pm 0.001}$ & $0.769_{\pm 0.001}$& $0.680_{\pm 0.270}$ & $0.126_{\pm 0.108}$ &$0.971_{\pm 0.040}$  & $0.717_{\pm 2.85}$ & $2.507_{\pm 1.957}$ 
    \\
    \midrule
    MLE  & {$0.065_{\pm 0.032}$} & {$0.034_{\pm 0.015}$} & $\bm{0.014_{\pm 0.002}}$ & $\bm{0.025_{\pm 0.032}}$&$\bm{0.045_{\pm 0.041}}$ & $\bm{0.006_{\pm 1.06}}$ & \bm{$0.051{\pm 0.049}$}\\
    \bottomrule
    % \midrule
    % Exp Hawkes & $-0.33_{\pm 0.03}$ & $-0.28_{\pm 0.05}$& $-0.32_{\pm 0.02}$ & $-0.35_{\pm 0.05}$ \\
    % \midrule
    % Sin Hawkes & $-1.11_{\pm 0.04}$ & $-1.04_{\pm 1.02}$ & $-0.74_{\pm 0.02}$ & $-0.90_{\pm 0.03}$\\
\end{tabular}}
\end{sc}
\end{table*}

\begin{figure}[t]
    \begin{center}
    \adjustbox{valign=b}{
    \begin{minipage}{0.45\linewidth}
    \includegraphics[width=\columnwidth]{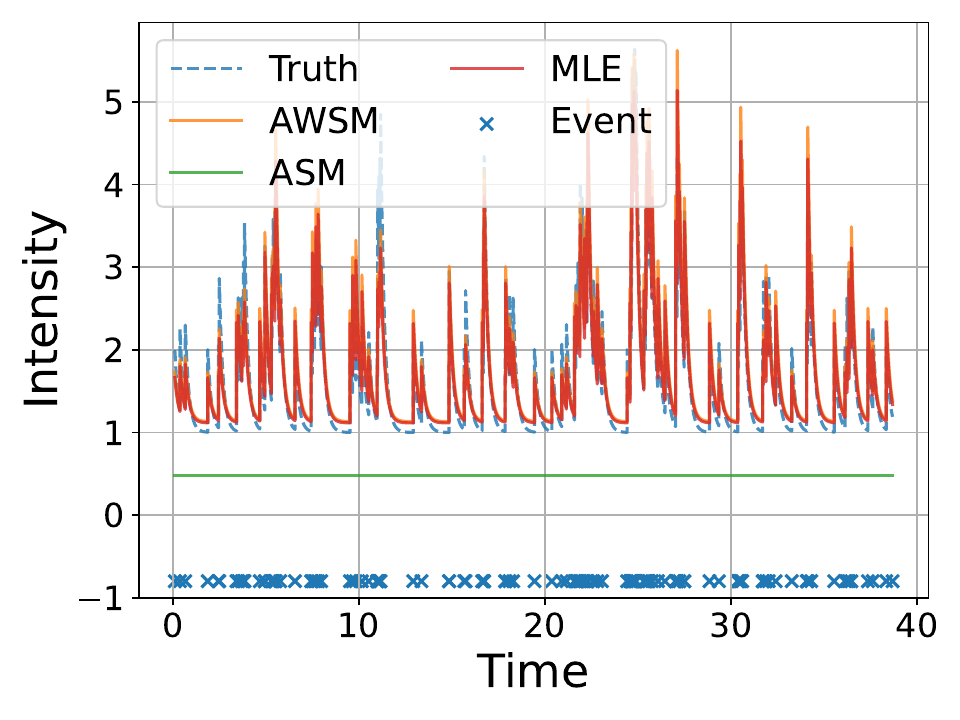}
    \subcaption{Exp-Hawkes Dim-2}
    %\label{fig: Sinusoidal Intensity}
    \end{minipage}}
    \adjustbox{valign=b}{
    \begin{minipage}{0.45\linewidth}
    \includegraphics[width=\columnwidth]{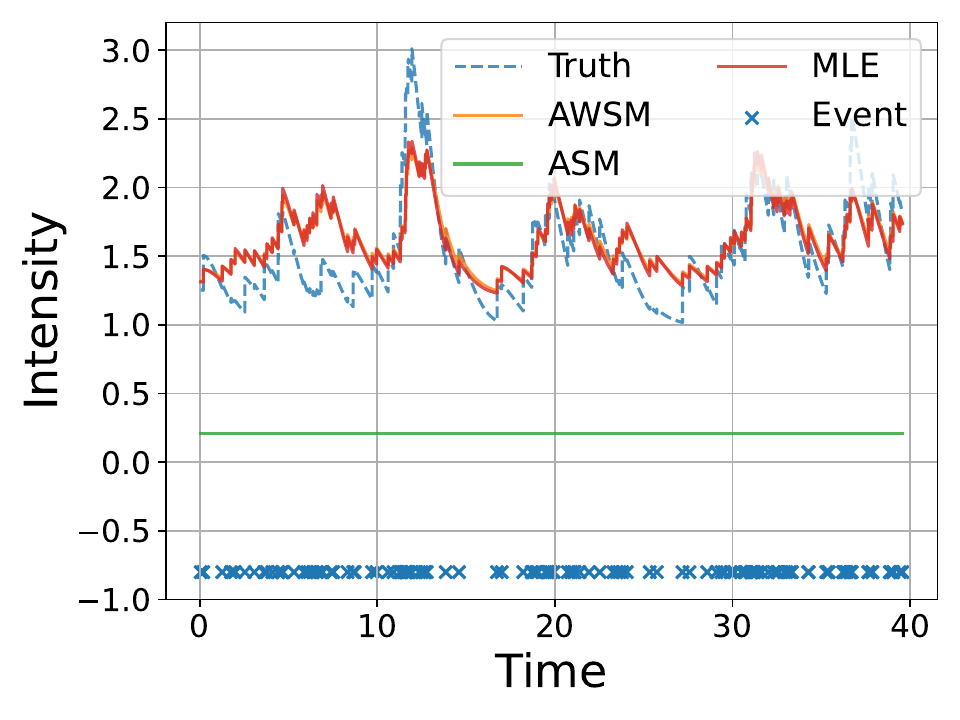}
    \subcaption{Gaussian-Hawkes Dim-2}
    %\label{fig: Sinusoidal Intensity}
    \end{minipage}}
    \end{center}
    \caption{The learned intensity functions from MLE, (A)SM, and (A)WSM on (a) Exp-Hawkes and (b) Gaussian-Hawkes for the 2-nd dimension. 
    }
    \label{fig:second_dim_hawkes}
\end{figure}

\begin{table*}[t]
\centering
\caption{The hyparameters for experiments in \cref{table: Real Data Experiment}.}
\begin{sc}
   \scalebox{0.8}{
    \begin{tabular}{l|cc|cc|cc|cc|cc}

\toprule
    \multirow{2}{*}{Dataset} & \multicolumn{2}{c}{Epochs} &\multicolumn{2}{c}{$\alpha_{\text{AWSM}}$}&\multicolumn{2}{c}{Trunc}&\multicolumn{2}{c}{$\alpha_{\text{DSM}}$}&\multicolumn{2}{c}{$\sigma_{\text{DSM}}$}\\
    \cmidrule{2-11}
        &SAHP &THP &SAHP&THP&SAHP&THP&SAHP&THP&SAHP&THP \\
        % \toprule
        %  \text{Dataset} & \text{Epochs} & \text{$\bm \alpha_{\text{AWSM}}$}&\text{Truncation} &$\bm \alpha_{\text{DSM}}$ & $\bm \sigma_{\text{DSM}}$ \\
         
        \midrule \text { \text{Half-Sin} } & $100$ &$100$ & $20$ &$20$ & F &F & $20$&10 &$0.01$&$0.01$\\
        \text { \text{STACKOVERFLOW} } &$100$ &$500$ &$20$ & $20$&T& F& $10$& 10&$0.1$&$0.1$\\
        \text { \text{TAOBAO} } &$500$  &$300$ &$50$ &$20$ &T &T & $20$& 20&0.01&5e-3\\
        \text { \text{RETWEETS } } &$100$ &$100$ &$20$&$10$ & F&F & $2000$& 10&0.01&5e-4 \\
        \bottomrule
    \end{tabular}}
\end{sc}\label{table:hyper}
\end{table*}

For parametric models, \cref{table:additional_experiments} provide the estimation results for some parameters and \cref{fig:second_dim_hawkes} shows the learned intensity functions on the 2-nd dimension for the 2-variate Hawkes processes. 

For deep point process experiments, we run 4 datasets with 3 methods deployed on 2 models. We show the hyperparameters of those experiments in \cref{table:hyper}. \textbf{Epochs} column represents the number of epochs for the experiments of a model (THP or SAHP) on a dataset. We train same epochs for three training methods (MLE, AWSM or DSM) and validate every 10 epochs to report the best result. When training MLE, the hyperparameter is the number of integral nodes, which is always 10 for all experiments. When training AWSM, we have two hyperparameters, the value of balancing coefficient for CE loss, shown in column $\alpha_{\text{AWSM}}$ and whether data truncation is performed as shown in column \textbf{Trunc}. T represents performing data truncation and F represents no data truncation. For DSM, the hyperparameters are balancing coefficient denoted as $\alpha_{\text{DSM}}$ and variance of noise denoted as $\sigma_{\text{DSM}}$.

\end{document}